\documentclass[11pt]{article}

\usepackage{mathrsfs,amsmath,amsfonts,amssymb,amstext,amscd,bm,bbm,dsfont,pifont, 
            amsthm,stmaryrd,euscript,color,xcolor,accents,xr} 

\usepackage[backref]{hyperref}

\usepackage{algorithmic}
\usepackage{algorithm}
\usepackage{url}
\usepackage{mathtools}
\usepackage{epsfig}
\usepackage{float}
\usepackage{appendix}
\usepackage{amstext}
\usepackage{floatflt}
\usepackage{nicefrac}
\usepackage{amsmath}
\usepackage{anysize,hyperref}
\usepackage{enumerate}
\usepackage{epsfig}
\usepackage{hyperref}
\usepackage{graphicx}
\numberwithin{equation}{section}

\input{macros_nicole}

\usepackage{mathtools}

\newcommand{\ra}{\rangle}

\usepackage[linecolor=magenta!60!, backgroundcolor=magenta!10!,textwidth=1.6cm, textcolor=magenta]{todonotes}

\title{Data splitting improves statistical performance \\ in overparameterized regimes}

\author{Nicole M\"ucke \\
Technical University  of Braunschweig \\
\texttt{nicole.muecke@tu-braunschweig.de} 
\and 
Enrico Reiss\\
University of Potsdam \\ 
\texttt{enreiss@uni-potsdam.de}
\and
Jonas Rungenhagen\\
University of Potsdam \\ 
\texttt{jrungenh@uni-potsdam.de}
\and 
Markus Klein\\
University of Potsdam \\ 
\texttt{mklein@math.uni-potsdam.de}
}
\date{\today}

\begin{document}

\maketitle

\begin{abstract}
While large training datasets generally offer improvement in model performance, the
training process becomes computationally expensive and time consuming. Distributed
learning is a common strategy to reduce the overall training time by exploiting multiple
computing devices. 
Recently, it has been observed in the single machine setting that
overparameterization is essential for benign overfitting in ridgeless regression in Hilbert spaces. We show that in this regime, 
data splitting has a regularizing effect, hence improving statistical performance and computational complexity at the same time. 
We further provide a unified framework that allows to analyze both the finite and infinite dimensional setting.  
We numerically  demonstrate the effect of different model parameters.  
\end{abstract}


\section{Introduction}

Modern machine learning applications often involve learning statistical models of great complexity and datasets of massive size 
become increasingly available. 
However, while increasing the size of the training datasets generally
offers improvement in model performance, the training process is very computation-intensive and thus time-consuming. 
Indeed, hardware architectures have physical limits in terms of storage, memory, processing speed and communication. 
A central challenge is thus to design efficient  large-scale algorithms.  

Distributed learning and parallel computing is a common and simple approach to deal with large datasets. 
The $n$ observations are evenly split to $M$ machines (or \emph{local nodes, workers}), each  having access to only a subset
of $n/M$ training samples. Each machine performs local computations to fit a model and transmits it to a central node 
for merging. This simple \emph{divide and conquer} approach having been proposed in e.g. \cite{mann2009efficient} for 
striking best balance between accuracy and communication is highly communication
efficient: Only one communication step is performed to only one central node\footnote{This approach is also called \emph{centralized learning}.}.    

The field of distributed learning has gained increasing attention in different regimes in the last years, with the aim of 
establishing conditions for the distributed estimator to be consistent or minimax optimal, 
see e.g. \cite{chen2014split}, \cite{mackey2011divide}, \cite{xu2019distributed}, \cite{fan2019distributed}, \cite{shi2018massive}, \cite{battey2018distributed}, 
\cite{fan2021communication}, \cite{bao2021one}. 
We give a more detailed overview over 
approaches that are most closely related to our approach. For a general overview we refer to \cite{bekkerman2011scaling} 
and the recent review \cite{gao2021review}.  

The learning properties of distributed (kernel) ridge regression are well understood. The authors in \cite{zhang2015divide} 
show optimal learning rates with appropriate regularization, 
if the number of machines increases sufficiently slowly with the sample size, though under restrictive assumptions on the eigenfunctions of the kernel 
integral operator. This has been alleviated in \cite{lin2017distributed}. However, in these works the number of machines \emph{saturates} 
if the target is very smooth, meaning that large parallelization seems not possible in this regime. This is somewhat counterintuitive as smooth signals are easier to reconstruct.      
To overcome this issue, the authors \cite{chang2017distributed} utilize additional unlabeled data, leading to a slight improvement. 

These works have been extended to more general spectral regularization algorithms  for nonparametric least square regression in 
(reproducing kernel) Hilbert spaces in \cite{guo2017learning}, \cite{mucke2018parallelizing}, including gradient descent \cite{lin2018distributed} 
and stochastic gradient descent \cite{lin2018optimal}.   

Finally, we mention \cite{zhang2013communication}, \cite{dobriban2021distributed}, \cite{rosenblatt2016optimality} who study averaged empirical risk minimization 
in the underparameterized regime, the latter in the high dimensional limit.

We consider distributed ridgeless regression over Hilbert spaces with (local) overparameterization.  
This setting has been investigated recently in e.g. \cite{bartlett2020benign}, \cite{chinot2020benign}, 
\cite{shang2021benign}, \cite{muthukumar2020harmless} in the single machine context with the aim of establishing conditions when \emph{benign} or \emph{harmless} 
overfitting occurs. This serves as a proxy to understand neural network learning 
where the phenomenon of benign overfitting was first observed \cite{bartlett2021deep, belkin2021fit}. Indeed, wide networks that are trained with gradient descent can be 
accurately approximated by linear functions in a Hilbert space. Our results are a step towards understanding the statistical effects in distributed settings  
in deep learning.   



{\bf Contributions.} 
{\bf 1.} We provide a unified framework that allows to simultaneously analyze the finite and infinite dimensional distributed ridgeless regression problem. 
All our bounds are optimal.   
\vspace{0.1cm}
\\
{\bf 2.} We show that in the presence of overparameterization the number of data splits has a regularizing effect that trades off bias and variance. 
While overparameterization induces an additional bias,  averaging reduces variance sufficiently. 
Hence, data splitting improves statistical accuracy (for an increasing number of splits until the optimal number is achieved) and scales to large data sets at once. 
Our approach fits into the line of \emph{communication efficient} distributed algorithms and is easy to implement. 
\vspace{0.1cm}
\\
{\bf 3.} To precisely quantify the interplay of statistical accuracy, computational complexity and signal strength we work in a general random-effects model. 
We find that the numerical speed up\footnote{In the sense that the optimal number of data splits is large and hence allows more parallelization.} 
is high for low signal strength and improves efficiency. A similar phenomenon is observed in \cite{sheng2020one} for distributed ridge regression. 
In addition, we do not observe a saturation effect for the number of machines as described above for kernel ridge regression. 
\vspace{0.1cm}
\\
{\bf 4.} The spectral properties of the covariance operator 
also highly impact the learning properties of distributed ridgeless regression. The spectral decay needs to be sufficiently fast for a high statistical accuracy.  
Note that this is known for the single machine setting from \cite{bartlett2020benign}.  



{\bf Organization.} 
In Section \ref{sec:setting} we define the mathematical framework needed to present our main results in  Section \ref{sec:main-results}. 
Section \ref{sec:discussion} is devoted to a discussion with a more detailed comparison to related work.  
Some numerical illustrations can be found in Section \ref{sec:numerics} while the Appendix contains all proofs and additional material. 

{\bf Notation.} 
By $\cL(\cH_1, \cH_2)$ we denote the space of bounded linear operators between real Hilbert spaces $\cH_1$, $\cH_2$. 
We write $\cL(\cH, \cH) = \cL(\cH)$. For $\Gamma \in \cL(\cH)$ we denote by $\Gamma^T$ the adjoint operator and for compact $\Gamma$ 
by $(\lam_j(\Gamma))_{j}$ the  sequence of eigenvalues. If $\beta \in \cH$ we write $\beta \otimes \beta := \langle \cdot, \beta \ra \beta$. 
We let $[n]=\{1,...,n\}$.  For two positive sequences $(a_n)_n$, $(b_n)_n$ we write $a_n \lesssim b_n$ if $a_n \leq c b_n$
for some $c>0$ and $a_n \simeq b_n$ if both $a_n \lesssim b_n$ and $b_n \lesssim a_n$.


\section{Setup}
\label{sec:setting}

In this section we provide the mathematical framework for our analysis. More specifically, we 
introduce distributed ridgeless regression and state the main assumptions on our model.

\subsection{Linear Regression}
\label{subsec:linear-regression}

We consider a linear regression model over a real separable Hilbert space $\cH$ in random design. More precisely, 
we are given a random covariate vector $x \in \cH$ and a random output $y \in \mbr$ following the model  
\begin{equation}
\label{eq:model}  
y = \inner{\beta^*, x} + \epsilon \;, 
\end{equation}
where $\epsilon \in \mbr$ is a noise variable. The true regression parameter $\beta^* \in \cH$ minimizes the least squares test risk, i.e.  
\[  \cR(\beta^*) = \min_{\beta \in \cH} \cR(\beta)\;, \qquad \cR(\beta) := \mbe[(y - \inner{\beta , x})^2] \;, \]
where the expectation is taken with respect to the joint distribution of the pair $(x,y) \in \cH \times \mbr$. 
This framework covers many common supervised learning tasks, e.g. learning in reproducing kernel Hilbert spaces \cite{RosVilla15}.

For our analysis we need to impose some distributional assumptions. To this end, we 
recall that a positive definite operator $\Gamma \in \cL(\cH)$ is \emph{trace class} (and hence compact), if 
\[ \tr(\Gamma ) = \sum_{j\in \mbn } \lam_j(\Gamma) < \infty \; ,  \]
see e.g. \cite{reed2012methods}. 


\vspace{0.2cm}
\begin{definition}[Hilbert space valued subgaussian random variable]
\label{def:subgaussian}
Let $z$ be a random variable in $\cH$ and let $\Gamma: \cH \to \cH$ be a bounded, linear and self-adjoint positive definite trace class operator. 
Given some $\sigma > 0$ we say that $z$ is $(\sigma^2, \Gamma)$-subgaussian if for all $v \in \cH$ one has 
\[  \mbe\left[ e^{\inner{v, z-\mbe[z]}} \right] \leq e^{\frac{\sigma^2}{2}\inner{\Gamma v , v}}  \;.\]
\end{definition}

Note that (taking $\cH=\mbr$)  this definition includes the special case of real valued variables. 
On $\cH$, we define the \emph{covariance operator} $\Sigma$
by $\Sigma u:=\mbe[ \inner{u, x}x]$, where $\mbe$ denotes expectation w.r.t.  the marginal distribution. We assume  

\begin{assumption}
\label{ass:design}
\begin{enumerate}
\item  $\mbe[x]=0$ and $\mbe[||x||^2]< \infty$. 
\item $x$ is $(\sigma_x^2, \Sigma)$-subgaussian and has independent components. 
\item The covariance $\Sigma$ possesses an orthonormal basis of eigenvectors $v_j$  with eigenvalues $\lam_1 \geq \lam_2 \geq ... $\; (counted according to multiplicity).
\item Conditionally on $x$, the noise $\eps$ in equation \eqref{eq:model} is centered and $(\tau^2, id)$-subgaussian, where $id$ denotes the identity on $\mbr$. 
\end{enumerate}
\end{assumption}

Note that 1. and 3. imply that $\Sigma$ is trace class (and also positive and self-adjoint). Indeed, this easily follows from 
\[ \mbe[||x||^2] = \sum_{j \in \mbn} \inner{v_j , \Sigma v_j} =   \sum_{j \in \mbn} \lam_j  < \infty \;. \]

To derive an estimator $\hat \beta \in \cH$ for $\beta^*$ we are given an i.i.d. dataset 
\[   D:= \{(x_1, y_1), ..., (x_n, y_n) \} \subset  \cH \times \mbr \;, \] 
following the above model \eqref{eq:model}, with i.i.d. noise 
$\beps= (\eps_1, ..., \eps_n) \in \mbr^n$. 
The corresponding random vector of outputs is denoted as $\bY=(y_1,\ldots, y_n)^T \in \mbr^n$ 
and we arrange the data $x_j \in \cH$ into a {\em data matrix} 
$\bX \in \cL(\cH, \mbr^n)$ by setting $(\bX v)_j =\langle x_j,v \rangle$ for $v \in \cH, 1 \leq j \leq n$. If $\cH=\mbr^d$, then $\bX$ is a $n \times d$ matrix
(with row vectors $x_j$).


\subsection{Distributed Ridgeless Regression}

In the distributed setting, our data are evenly divided into $M$ local disjoint subsets  
\[ D = D_1 \cup ... \cup D_M   \] 
of size $|D_m|=\frac{n}{M}$, for $m=1,...,M$. To each local dataset we associate 
a \emph{local design matrix} $\bX_m \in \cL(\cH,\mbr^{\frac{n}{M}})$ with local output vector
$\bY_m \in \mbr^{\frac{n}{M}}$ and a local noise vector $\beps_m \in \mbr^{\frac{n}{M}}$.

In addition to the above distributional assumptions we require: 

\begin{assumption}
\label{ass:eigen-spaces}
Let $m \in [M]$. 
Almost surely, the projection of the local data $\bX_m$ on the space orthogonal to any eigenvector of $\Sigma$ 
spans a space of dimension $\frac{n}{M}$. 
\end{assumption}

More precisely, recall that the data matrix $\bX_m$ is built up from $n/M$ {\em row vectors} $x_k \in \cH$. 
The above assumption means that those row vectors almost surely are 
in {\em general position}: Only with zero probability the orthogonal projections of those vectors are linearly dependent 
in each hyperplane $H_j:=\{ x \in \cH;\langle x,v_j \rangle =0\}$ orthogonal  to the 
eigenvector $v_j$ of $\Sigma$. In particular, data vectors $x_j$ are collinear to some $v_j$ with zero  probability.

Note that Assumptions \ref{ass:design}  and \ref{ass:eigen-spaces} are satisfied if $x,y$ are jointly gaussian with zero mean and rank$(\Sigma)>n/M$.

We define the \emph{local minimum norm estimator} $\hat \beta_m$ as the solution to the 
optimization problem 
\begin{align*}
\min_{\beta \in \cH} & \;\; ||\beta ||^2 \quad \mbox{ such that } \\
||\bX_m \beta - \bY_m||^2 &= \min_{\tilde \beta \in \cH} ||\bX_m \tilde \beta - \bY_m ||^2 \;.
\end{align*}

It is well known that $\hat \beta_m$  has a closed form expression (see \cite{engl1996regularization}) given by 
\begin{equation}
\label{eq:ridgeless-2}
 \hat \beta_m =   \bX_m^T (\bX_m \bX_m^T)^{\dagger}  \bY_m \;,  
\end{equation} 
where $(\bX_m \bX_m^T)^{\dagger}$ denotes the pseudoinverse of the bounded linear operator $\bX_m \bX_m^T$.

In the case that dim$(\cH)=d<\frac{n}{M}$ and $\bX_m$ has rank $d$, there is a unique solution to the normal equations. 
However, under Assumption \ref{ass:eigen-spaces} we find many local interpolating solutions $\beta \in \cH$ to the normal equations 
with $\bX_m \beta = \bY_m$.

The final estimator is defined as the uniform average  
\begin{equation}
\label{eq:weighted-ave}
  \bar \beta_{M}  = \frac{1}{M} \sum_{j=1}^M  \beta_m \;.  
\end{equation}

We aim at finding optimal bounds for the excess risk 
\[ \cR(\bar \beta_M) - \cR(\beta^*) = ||\Sigma^{\frac{1}{2}}(\bar \beta_M - \beta^*) ||^2\;, \]
in high probability, as a function of the number of local nodes $M$ and under various model assumptions.


\section{Main Results}
\label{sec:main-results}

In this section we state our main results. We first derive a general upper bound and consider 
the infinite and finite dimensional settings in more detail. We complete our presentation 
with matching lower bounds.


\subsection{A General Error Bound}
\label{subsec:general-upper-bound}

Before stating our error bounds we briefly describe the underlying error decomposition in bias and variance. 
For an estimator $\hat \beta \in \cH$ let us define the \emph{bias} by 
\[ \bias (\hat \beta) := ||\Sigma^{1/2} (\mbe_\epsilon[\hat \beta] - \beta^* )||^2   \]
and the \emph{variance} as 
\begin{equation*}
\var (\hat \beta) := \mbe_\epsilon \left[||\Sigma^{1/2} (\hat \beta -  \mbe_\epsilon[\hat \beta ] ) ||^2  \right]  \;,
\end{equation*}
where $\mbe_\eps[\cdot]$ denotes the conditional expectation given the input data. 
We then have the following preliminary bound for the excess risk whose full proof is given in Appendix \ref{supp:proofsA}.

\begin{lemma}
\label{lem:bias-var}
Let $\bar \beta_{M}$ be defined by \eqref{eq:weighted-ave} and denote by 
$\hat \Sigma_m = \frac{M}{n} X_m^T X_m$ 
the local empirical covariance operator.  
The excess risk can be bounded almost surely by
\[ \mbe_\epsilon \left[||\Sigma^{\frac{1}{2}}(\bar \beta_M - \beta^*) ||^2 \right]  = \bias (\bar \beta_{M}) + \var (\bar \beta_{M}) \;, \]
where  
\[ \bias(\bar \beta_{M}) \leq  \frac{1}{M}  \sum_{m=1}^M  \left| \inner{\beta^* , (\Sigma - \hat \Sigma_m ) \beta^*} \right| \;, \]
\[ \var (\bar \beta_{M}) \leq  \frac{8 \tau^2}{M^2}  \sum_{m=1}^M \tr\left[ \left( X_m^\dagger \right)^T \Sigma X_m^\dagger \right] \;. \]
\end{lemma}

We are interested in finding conditions such that bias and variance (and thus the excess risk) converge to zero with high probability. To this end, 
we also take the hardness of the learning problem into account. This can be quantified via a classical a-priori 
assumption on the minimizer $\beta^*$.

\begin{assumption}[General random-effects model]
\label{ass:source}
Let $\Theta \in \cL(\cH)$ be compact. 
Let $\beta^*$ be randomly sampled (independently of $\eps$) with mean $\mbe_{\beta^*}[\beta^*] = 0$ and covariance $\mbe_{\beta^*}[ \beta^* (\beta^*)^T ] = \Theta$. 
\end{assumption}

This assumption is a slight generalization of the classical concept of a \emph{source condition} in inverse problems \cite{mathe2003geometry} and learning in 
(reproducing kernel) Hilbert spaces  \cite{bauer2007regularization, blanchard2018optimal, lin2020optimal}; see also \cite{richards2020asymptotics}, \cite{sheng2020one}  
for the context of (distributed) high dimensional ridge(less) regression. We give some specific examples in Assumptions \ref{ass:sparse}, \ref{ass:source-hoelder} below.


For bounding the variance we follow the approach in \cite{chinot2020benign}, \cite{bartlett2020benign} and choose an index $k \in \mbn$  
and split the spectrum of $\Sigma$ accordingly.  
For a suitable choice of $k$ (called \emph{effective dimension})
it will be crucial to control 
two notions of the \emph{effective ranks},  see e.g. \cite{koltchinskii2017concentration, bartlett2020benign} 

\begin{definition}[Effective Ranks]
\label{def:eff-ranks}
For $k \geq 0$ with $\lam_{k+1}>0$ we define 
\[  r_k(\Sigma) := \frac{\sum_{j>k}\lam_j(\Sigma)}{\lam_{k+1}(\Sigma)} \;, 
\quad R_k(\Sigma) = \frac{ \left(\sum_{j>k}  \lam_j(\Sigma) \right)^2 }{\sum_{j>k}  \lam_j^2(\Sigma) } \;.\]
\end{definition}

\begin{definition}[Effective Dimension]
\label{def:eff-dim}
Let $a>1$ and $M\in [n]$. Define the \emph{effective dimension} as 
\[ k^*=k^*_{\frac{n}{M}} := \min\left\{ k\geq 0\;:\; r_k(\Sigma) \geq a \frac{n}{M} \right\}  \;, \]
where the minimum of the empty set is defined as $\infty$.
\end{definition}

Our main result gives an upper bound for the bias and variance in terms of the source condition, effective ranks and effective dimension.  

\begin{theorem}
\label{theo:main1}
Suppose Assumptions \ref{ass:design}, \ref{ass:eigen-spaces} and \ref{ass:source} are satisfied and let $\delta \in (0,1]$. 
There exists a universal constant $c_1>0$  
such that for all $\frac{n}{M} \geq \frac{1}{c_1}\log(2/\delta)$, 
with probability at least $1-\delta$
\[ \mbe_{\beta^*}[\bias(\bar \beta_{M})] 
\leq  \frac{4 \sigma_x}{c_1} \log^{\frac{1}{2}}\left(\frac{2M}{\delta}\right) \tr[ \Sigma \Theta]   \sqrt{\frac{M}{n}} \;.   \]
Additionally, there exist $c_2>1$ such that, if 
\[ k_{\frac{n}{M}}^* \leq \frac{n}{c_2M} \;, \] 
with probability at least $1-7Me^{-\frac{n}{c_2 M}}$
\begin{equation}
\label{eq:upper-variance}
  \var (\bar \beta_M) \leq  8c_2 \tau^2  \left( \frac{k_{\frac{n}{M}}^*}{n} +  
\frac{n}{M^2}\;\frac{1}{R_{k_{\frac{n}{M}}^*}(\Sigma)}  \right) \;.  
\end{equation}
\end{theorem}

Theorem \ref{theo:main1} reveals that the excess risk of the averaged local interpolants converges to zero if 
\[ \tr[ \Sigma \Theta]\sqrt{\frac{M_n}{n}} \to 0 \;, \;\; \frac{k_{\frac{n}{M_n}}^*}{n} \to 0\;,  \]
\[ \frac{n}{M_n^2}\;\frac{1}{R_{k_{\frac{n}{M_n}}^*}(\Sigma)} \to 0\;, \]
for $M_n \leq n$. This imposes restrictions on the decay of the eigenvalues of $\Sigma$. Moreover, the convergence of the bias 
depends on the prior assumption on $\beta^*$.

In the following two subsections we discuss the infinite dimensional and finite dimensional cases in more detail.


\subsection{Infinite Dimension}
\label{subsec:inf-dim}

We refine the excess risk bound under more specific assumptions on $\beta^*$ and the spectral decay of the covariance. 

{\bf Source Condition.} The a-priori assumption on $\beta^*$ from Assumption \ref{ass:source} can be expressed via an   
increasing \emph{source function} $\Phi: \mbr_+ \to \mbr_+$ by setting $\Theta = \Phi(\Sigma)$,   
describing how coefficients of $\beta^*$ vary along the eigenvectors of $\Sigma$, see e.g. \cite{richards2020asymptotics}. 
Recall that the bias in Theorem \ref{theo:main1} depends on 
\[ \tr[ \Sigma \Theta] = \tr[ \Sigma \Phi(\Sigma) ] = \sum_{j=1}^\infty \lam_j \Phi(\lam_j)\;.  \]
Thus, the bias is finite if the map $x \mapsto x \Phi(x)$ is non-decreasing while the sequence of eigenvalues $(\lam_j)_{j \in \mbn}$ is decreasing.

\begin{assumption}[Source Condition]
\label{ass:source-hoelder}
Assume  that $\Phi(x)=x^\alpha$, for $\alpha \geq 0$. 
\end{assumption}

This particular choice of source function goes under the name \emph{H\"older-type source condition} and is a standard assumption  
in inverse problems \cite{mathe2003geometry} and nonparametric regression  \cite{bauer2007regularization, blanchard2018optimal, lin2020optimal}. 
Indeed, it has a direct characterization in terms of \emph{smoothness}, where a larger exponent $\alpha$ corresponds to a smoother regression function. 
In this regard, this assumption also quantifies the \emph{easiness} of the learning problem: Larger values of $\alpha$ indicate an easier problem, as 
smoother functions are easier to recover.

{\bf Eigenvalue Decay.} 
Finally, to control the variance in Theorem \ref{theo:main1} we impose a specific spectral assumption for 
the covariance: 

\begin{assumption}
\label{ass:poly-decay}
Assume  that $\lam_j(\Sigma)=j^{-(1+\eps_n)}$ for some positive sequence 
$(\eps_n)_{n \in \mbn}$ with $M_n \lesssim \eps_n n$.
\end{assumption}

Polynomially decaying eigenvalues are a common assumption in ridgeless regression. Indeed, it is shown for the single machine setting 
in \cite{bartlett2020benign} that under this assumption, the excess risk of the least-norm interpolant converges to zero and \emph{benign overfitting} occurs. 
   
\vspace{0.1cm}

Our main result in this section is a refined upper bound for the excess risk under the two additional assumptions made above.  
The proof is given in Appendix \ref{app:special-cases}.

\vspace{0.1cm}

\begin{proposition}
\label{prop:poly}
In addition to all assumptions of Theorem \ref{theo:main1}, suppose that Assumptions \ref{ass:poly-decay}, \ref{ass:source-hoelder} hold. 
Set 
\[ C_{\alpha, n}= \frac{1}{\alpha}\eins\{ \alpha> 0\} + \frac{1}{\eps_n}\eins\{ \alpha=0 \} \]
and assume that $\frac{n}{M} \geq \frac{1}{c^2_1} \log(2/\delta) $. 
With probability at least $1-\delta-7Me^{-\frac{n}{c_2 M}}$ we have 
\begin{align*}
 \mbe_{\beta^* ,\epsilon} \left[||\Sigma^{\frac{1}{2}}(\bar \beta_M - \beta^*) ||^2 \right] 
&\leq c_3 \sigma_x  \log^{\frac{1}{2}}\left(\frac{2M}{\delta}\right) C_{\alpha, n} \sqrt{\frac{M}{n}}   + 
 c_4\tau^2 \frac{\eps_n}{M} \;, 
\end{align*}
for some $c_3>0$, $c_4>0$. 
\end{proposition}

The above result offers the following insights: 
\vspace{0.1cm}
\\
{\bf 1.} 
The dependence of our error approximations on the
number of machines reveals an interesting accuracy-complexity trade-off. Indeed, 
data splitting has a regularizing effect, where the number of local nodes $M$ acts as an explicit regularization parameter: 
The bias term is increasing as $\sqrt{M}$ 
while the variance is decreasing as $1/M$. 
\vspace{0.1cm}
\\
{\bf 2.}  The source condition controls the bias: The smoother the solution, i.e. the larger $\alpha >0$, the smaller the bias. 
Notably, we observe a \emph{phase transition} to the case where $\alpha=0$ (low smoothness, harder problem). The bias is multiplied by a factor $1/\eps_n$ for a 
sequence $\eps_n \to 0$ and hence grows with $n$ while for $\alpha >0$ the factor is $1/\alpha$ that is constant in $n$ and decreasing with $\alpha$.   
\vspace{0.1cm}
\\
{\bf 3.} Eigenvalue decay, reflected in the sequence $(\eps_n)_n$ controls the variance: Ideally, we want $\eps_n \to 0$ to 
achieve fast decay of the variance. However, even increasing $(\eps_n)_n$ is possible as long as we ensure that $\eps_n/M_n \to 0$. 
\vspace{0.1cm}
\\
Balancing both terms  allows to establish learning rates for different smoothness regimes (see Appendix \ref{app:special-cases}):  

\begin{corollary}[Learning rate high smoothness]
\label{cor:poly1}
Suppose all assumptions of Proposition \ref{prop:poly} are satisfied and let $\alpha >0$. 
For 
\begin{equation}  
\label{eq:eps1}
 \frac{1}{\sqrt n} \lesssim \eps_n \lesssim n\;,  
\end{equation}
the value 
\[  M_n = C_{\tau, \sigma_x}\left( \alpha \eps_n \sqrt n\right)^{2/3} \]
with $C_{\tau, \sigma_x} = \left( \frac{c_4\tau^2}{c_3 \sigma_x} \right)^{2/3} $ 
trades-off bias and variance and 
with the same probability as above, we have 
\begin{align}
\label{eq:rate-1}
  \mbe_{\beta^* ,\epsilon} \left[||\Sigma^{\frac{1}{2}}(\bar \beta_{M_n} - \beta^*) ||^2 \right] 
&\leq  \frac{C'_{\tau, \sigma_x}}{\alpha^{2/3}}\log^{\frac{1}{2}}\left(\frac{4M_n}{\delta}\right)  \left( \frac{ \eps_n}{n} \right)^{1/3} \;, 
\end{align}
for some $C'_{\tau, \sigma_x}>0$.
\end{corollary}

\vspace{0.2cm}

\begin{corollary}[Learning rate low smoothness]
\label{cor:poly2}
Suppose all assumptions of Proposition \ref{prop:poly} are satisfied and let $\alpha =0$. 
For 
\begin{equation}  
\label{eq:eps2}
  \frac{1}{\sqrt n} \lesssim \eps^2_n \lesssim n\;,  
\end{equation}  
the value
\[  M_n = C_{\tau, \sigma_x}\left( \eps^2_n \sqrt n\right)^{2/3} \]
with $C_{\tau, \sigma_x} = \left( \frac{c_4\tau^2}{c_3 \sigma_x} \right)^{2/3} $ 
trades-off bias and variance and 
with the same probability as above, we have 
\begin{align*}
  \mbe_{\beta^* ,\epsilon} \left[||\Sigma^{\frac{1}{2}}(\bar \beta_{M_n} - \beta^*) ||^2 \right] 
&\leq  C'_{\tau, \sigma_x}\log^{\frac{1}{2}}\left(\frac{4M_n}{\delta}\right) \left( \frac{ 1}{\eps_n n} \right)^{1/3} \;, 
\end{align*}
for some $C'_{\tau, \sigma_x}>0$.
\end{corollary}



\subsection{Finite Dimension}
\label{subsec:sparse-finite-dim}

In this section we investigate the finite dimensional setting in more detail 
and assume $dim(\cH)=d<\infty$, where we are mostly interested in the global overparameterized case $d> n$. 
To highlight the effects of all characteristics effecting model 
performance, we make two particularly simple structural assumptions. More specifically, we assume the covariance $\Sigma$ 
to follow a \emph{strong and weak features model}:     

\begin{assumption}[Strong-weak-features model]
\label{ass:stron-weak}
Let $F \in [d]$ and $\rho_1 \geq \rho_2 > 0$. Suppose that $\lam_j(\Sigma) = \rho_1$ for all $j \in [F]$ and $\lam_j(\Sigma) = \rho_2$ 
for all $F+1 \leq j \leq d$. Without loss of generality, we assume that $||\Sigma||=1$, i.e. $\rho_1=1$.
\end{assumption}

Elements in the eigenspace associated to the larger eigenvalue $\rho_1$ are called \emph{strong features} while 
elements in the eigenspace associated to the smaller eigenvalue are called \emph{weak features}, see e.g. \cite{richards2020asymptotics}.

Furthermore, 
we work in a standard \emph{random-effects model}, see \cite{sheng2020one}, \cite{dobriban2018high}, \cite{dicker2017flexible}.   

\begin{assumption}[Random-effects model]
\label{ass:sparse}
Define the \emph{signal-to-noise-ratio} as 
\[ \SNR=\mbe[||\beta^*||^2]/\tau^2 \;. \] 
The coordinates of $\beta^*$ are independent, have zero mean and variance $\frac{\SNR}{d}$, i.e. 
$\Theta = \frac{\SNR}{d}\;  Id_d$. 
\end{assumption}

The next result presents an upper bound for the excess risk under both assumptions. 
The proof is provided in Appendix \ref{sec:finite-dim-sparsity}.

\begin{proposition}
\label{prop:sparse-strong-weak}
In addition to all assumptions of Theorem \ref{theo:main1}, suppose 
Assumptions \ref{ass:sparse}, \ref{ass:stron-weak} hold.  Assume that the weak features satisfy $\rho_2 d \leq F$ and that 
\begin{equation}
\label{eq:cond-F}
 a \frac{n}{M} <  \left(1-\rho_2 \right) F + \rho_2 d \;.
\end{equation} 
If $ \frac{n}{M} \geq \frac{1}{c_1}\log(2/\delta)$,  
then with probability at least $1-\delta-7Me^{-\frac{n}{c_2 M}}$, the excess risk satisfies 
for some $c, \tilde c>0$
\begin{align}
\label{eq:rhs}
 \mbe_{\beta^* ,\epsilon} \left[||\Sigma^{\frac{1}{2}}(\bar \beta_M - \beta^*) ||^2 \right] 
&   \leq    c\sigma_x \frac{ \SNR }{d}\;F  \log^{\frac{1}{2}}\left(\frac{2M}{\delta}\right)   \sqrt{\frac{M}{n}} 
 +    \tilde c \tau^2 C_{\rho_2} \; \frac{n}{M^2}\;\frac{1}{F}  \;, 
\end{align}
where we set $C_{\rho_2}=\frac{ 1}{( 1-\rho_2)^2}$. 
\end{proposition}

We comment on two additional assumptions required: First, \eqref{eq:cond-F} ensures 
sufficient local overparameterization and thus local interpolation. Indeed, assume that the number of strong features is a fraction of the dimension, i.e.  
$F=d/K$ for some $K>1$, then  \eqref{eq:cond-F} can be rewritten as  
\[   1\leq \frac{1}{ \tilde c} < \frac{d}{n/M} \;, \]
where for $K$ large enough 
\[\tilde c =  \frac{1}{a}\left( \frac{1}{K} + \rho_2 \left( 1- \frac{1}{K}\right) \right) \leq  1 \;. \]
Second, the assumption $\rho_2 d \leq F$ ensures that the strength $\rho_{2,n}$ of weak features is small enough and consequently they do not contribute much, while 
the amount $F$ of strong features is sufficiently high.

The above result shows how the various model parameters determine statistical accuracy: 
\vspace{0.1cm}
\\
{\bf 1.} As above, we observe that data splitting has a regularizing effect:  
The bias term is increasing as $\sqrt{M}$ 
while averaging significantly  reduces the variance by $1/M^2$. 
\vspace{0.1cm}
\\
{\bf 2.} The signal-to-noise ratio $\SNR$ and the ratio $F/d \leq 1$ of the number of strong features to the dimension control the bias: 
The bias is small, if both quantities are small. 
\vspace{0.1cm}
\\
{\bf 3.} The variance is controlled by the number $F$ of strong features. The more strong features, the faster the variance decreases.   
\\
\noindent

Thus, the excess risk is characterized by the interplay of all these parameters.   
Minimizing the rhs in \eqref{eq:rhs} in $M$ allows to trade-off these different contributions. 

\vspace{0.1cm}

\begin{corollary}[Optimal number of nodes]
\label{cor:sparse-strong-weak}
Suppose all assumptions of Proposition \ref{prop:sparse-strong-weak} are satisfied. 
Let $(\rho_{2,n})_n$ be decreasing and assume 
\begin{equation}
\label{eq:ass-f-d}
 \frac{\SNR}{n^{3/2}}  \lesssim    \frac{d}{F^2}  \lesssim  \SNR\; n  \;.  
\end{equation} 
The optimal 
number\footnote{The optimal number is defined as the minimizer of the right hand side in $\eqref{eq:rhs}$, where we ignore the log-factor.} of local 
nodes $M_n$ is given by 
\begin{equation}
\label{eq:nb-machnines}
 M_n= A\; \left(  \frac{ d n^{3/2}}{\SNR \cdot F^2} \right)^{2/5}\;, 
\end{equation} 
where $A =  c' \; \left( \frac{1}{1-\rho_{2,n}}   \right)^{4/5} $, 
for some $ c' > 0$, depends on $\tau^2, \sigma_x^2$.  
The excess risk satisfies  with probability at least $1-\delta-7M_n e^{-\frac{n}{c_2 M_n }}$
\begin{align}
\label{eq:rate-finite}
& \hspace{-0.8cm}  \mbe_{\beta^* ,\epsilon} \left[||\Sigma^{\frac{1}{2}}(\bar \beta_{M_n} - \beta^*) ||^2 \right] \no \\
& \;\;\;\; \leq  c'' \log^{\frac{1}{2}}\left(\frac{4M_n}{\delta}\right) \;  \frac{n}{F M_n^2} \;,
\end{align}
for some $c'' >0$.
\end{corollary}

The proof of Corollary \ref{cor:sparse-strong-weak} is given in Appendix \ref{sec:finite-dim-sparsity}. We comment on the above result in more detail. 

The optimal number $M_n$ of local nodes depends on various model parameters: 
We immediately observe that a large number $F$ of strong features decreases the error while it also decreases 
the optimal number of machines (recall that the larger $M_n$, the more computational savings). 
Notably, $M_n$ additionally depends on the spectral gap $1 - \rho_2$ showing that  
computational savings are enforced with a small spectral gap, see Fig. \ref{fig:1}. 
Moreover,  we find that the numerical speed up  is high for a low $\SNR$ and thus improves efficiency. 
A similar phenomenon is observed in \cite{sheng2020one} for distributed ridge regression in finite dimension.
  
{\bf When does the error converges to zero?} 
We consider now a \emph{high dimensional} and \emph{infinite-worker} limit. 
More specifically, we let $n \to \infty$, $d_n \to \infty$, $M_n \to \infty$. We are interested in the case of overparameterization, i.e. 
$n \lesssim d_n$. 
Recall that $M_n$ cannot grow faster that $n$ (otherwise there would be less than one sample per machine). 
For this to hold, \eqref{eq:nb-machnines} imposes $ n \lesssim d_n \lesssim n\cdot F^2_n $. 
Note we also have to require $F_n \lesssim d_n$. Thus, 
\[ M_n \simeq \left( \frac{d_n}{n} \right)^{2/5} \; \left( \frac{n}{F^{4/5}_n} \right)\;. \]
Hence, $M_n \to \infty$ if $n/F_n^{4/5} \to \infty$ and $1 \lesssim d_n/n$. 
These assumptions are satisfied for e.g. $d_n\simeq n^{\gamma}$, $F_n \simeq n^\delta$ with $\gamma >1$, $\delta \in (0, 5/4)$ 
and $\max\{1, \delta, 2\delta - 3/2 \} \leq \gamma \leq 2\delta +1$.  
The learning rate in this case is 
\begin{align}
\label{eq:to-zero}
&\lesssim \;\left( \frac{F_n}{d_n}\right)^{4/5}\left( \frac{1}{nF_n}  \right)^{1/5}\no \\
&\lesssim \left( \frac{1}{nF_n}  \right)^{1/5}\no \\ 
&\simeq \left( \frac{1}{n}  \right)^{(1+\delta)/5}\no \\
&\to 0 \;, 
\end{align}
converging to zero.


\subsection{Lower Bound}
\label{subsec:lower-bound}
 
Finally, we give a matching lower bound for the excess risk for the distributed 
estimator with the optimal choice of local nodes. 
All proofs of this section are provided in Appendix \ref{app:prof-lower-bound}. \\ 
The derivation of our result requires a lower bound for the noise variance: 
 
\begin{assumption}
\label{ass:lower-noise}
The conditional noise variance is almost surely bounded below by some constant $\sigma^2 >0$, i.e. 
$\mbe[ \eps^2 | x ] \geq \sigma^2$. 
\end{assumption}

We start with  a general lower bound for the excess risk in terms of the effective ranks and the 
effective dimension.   

\begin{theorem}
\label{theo:lower-bound}
Suppose Assumptions \ref{ass:lower-noise},\ref{ass:design}, \ref{ass:eigen-spaces} are satisfied. 
With probability at least $1-10Me^{-\frac{1}{c}\frac{n}{M}}$ 
\begin{align}
\label{eq:low}
 \mbe_\eps[ ||\Sigma^{1/2}( \bar \beta_M - \beta^*) ||^2 ] 
&  \geq  
 c_a\sigma^2  \left(  \frac{k^*_{\frac{n}{M}}}{n} + \frac{n}{M^2} \frac{1}{R_{k^*}(\Sigma)}  \right) \;,
\end{align} 
for some $c_a>0$.
\end{theorem}
 
Note that the lower bound for the excess risk is of the order of the variance bound \eqref{eq:upper-variance}.  
We emphasize that the optimal number $M_n$ of splits is derived by trading-off bias and variance. Hence, 
for this value, the bound \eqref{eq:low} is optimal. We give now the explicit optimal rates in the special settings 
from Sections \ref{subsec:inf-dim}, \ref{subsec:sparse-finite-dim}.

\begin{corollary}[Optimal rate infinite dimension]
\label{cor:opt-poly}
Suppose all Assumptions of Theorem \ref{theo:lower-bound} are satisfied. Let $n$ sufficiently large and 
recall the definition of $M_n$ from Corollary \ref{cor:poly1}. 
With probability at least $1-10M_ne^{-\frac{1}{c}\frac{n}{M_n}}$, the excess risk 
is lower bounded by 
\begin{align*}
\mbe_\eps[ ||\Sigma^{1/2}( \bar \beta_{M_n} - \beta^*) ||^2 ] 
&\geq \frac{\tilde C_{\tau, \sigma_x , \sigma}}{\alpha^{2/3} } \left( \frac{\eps_n}{n}\right)^{1/3} \;,
\end{align*}
for some $\tilde C_{\tau, \sigma_x , \sigma} > 0$. Hence, under the Assumptions of Corollary \ref{cor:poly1}, 
the rate of convergence is optimal (up to a log-factor) as it matches the upper bound \eqref{eq:rate-1}. 
Note that we also obtain the optimal bound from Corollary \ref{cor:poly2} in the low smoothness regime (see Appendix \ref{app:prof-lower-bound}).
\end{corollary} 
 
\vspace{0.1cm}

\begin{corollary}[Optimal rate finite dimension]
\label{cor:opt-finit-dim}
Recall the strong-weak-features model from Section \ref{subsec:sparse-finite-dim} and 
suppose all Assumptions of Theorem \ref{theo:lower-bound} are satisfied. 
With probability at least $1-10Me^{-\frac{1}{c}\frac{n}{M}}$, the excess risk 
is lower bounded by 
\begin{align*}
\mbe_\eps[ ||\Sigma^{1/2}( \bar \beta_{M} - \beta^*) ||^2 ] 
&\geq \tilde c_a\sigma^2  \tilde C_{\rho_2} \frac{n}{F M^2} \;,
\end{align*}
for some $\tilde  C_{\rho_2}  > 0$, $\tilde c_a > 0$. Moreover, under the assumptions of Corollary \ref{cor:sparse-strong-weak}, 
this lower bound matches the upper bound for the optimal $M_n$ 
and hence is optimal (up to a log-factor).  
\end{corollary}


\section{Remarks about efficiency}
\label{app:efficiency}

In addition to the non-asymptotic bounds on bias and variance we are interested in the possible gain in efficiency of data splitting compared 
to the single machine setting. To this end, let us introduce the ratio of the excess risks for the single machine estimator $\bar \beta_{1}$ 
and the distributed estimator $\bar \beta_{M}$, $M >1$.

\begin{definition}
\label{def:rel-eff-intro}
We define the \emph{relative prediction efficiencies} by 
\begin{align*}
 {\rm \widehat{Eff} }(M) &= 
\frac{\mbe_\epsilon \left[ ||\Sigma^{1/2} (\bar \beta_{1} - \beta^* )||^2\right]}{\mbe_\epsilon \left[ ||\Sigma^{1/2} (\bar \beta_{M} - \beta^* )||^2\right]} \;.
\end{align*}
\end{definition}

\subsection{Quadratic increase in efficiency in finite dimension}

We consider the setting of Section \ref{subsec:sparse-finite-dim}. To bound the relative prediction efficiency in this case recall that from 
Proposition \ref{prop:sparse-strong-weak} and Corollary \ref{cor:opt-finit-dim} we have in the single machine setting a lower and upper bound for 
the excess risk\footnote{We omit logarithmic terms.}: With probability at least $1-10e^{-\frac{n}{c}}$
\begin{align*}
   \frac{n}{F} \lesssim  \mbe_{\beta^* ,\epsilon} \left[||\Sigma^{\frac{1}{2}}(\bar \beta_1 - \beta^*) ||^2 \right] \lesssim  
\max\left\{ \frac{ F}{d \sqrt{n}} , \frac{n}{F}\right\}   \;. 
\end{align*}
Note that $\max\left\{ \frac{ F}{d \sqrt{n}} , \frac{n}{F}\right\} =  \frac{n}{F}$ if $\frac{d}{F^2}\gtrsim n^{-3/2}$. In this case, the bound is optimal.

Moreover, Proposition \ref{prop:sparse-strong-weak} and Corollary \ref{cor:opt-finit-dim} show that 
with probability at least $1-10M_{opt}e^{-\frac{1}{c}\frac{n}{M_{opt}}}$, the excess risk in the optimally distributed setting 
enjoys the optimal bound   
\[ \mbe_\eps[ ||\Sigma^{1/2}( \bar \beta_{M_{opt}} - \beta^*) ||^2 ]  \simeq  \frac{n}{F M_{opt}^2} \;, \]
where  we denote by $M_{opt}$ the optimal number of splits from \eqref{eq:nb-machnines}.

As a result, we obtain:

\begin{corollary}
If all assumptions of Proposition \ref{prop:sparse-strong-weak}, Corollary \ref{cor:sparse-strong-weak} and Corollary \ref{cor:opt-finit-dim} 
are satisfied, the relative prediction efficiency increases quadratically in the number of optimal splits, 
with probability at least $1-10M_{opt}e^{-\frac{1}{c}\frac{n}{M_{opt}}}$ 
\[  {\rm \widehat{Eff} }(M_{opt})  \simeq  M_{opt}^2 \;.\]
\end{corollary}

\vspace{0.3cm}

\subsection{Linear increase in efficiency in infinite dimension} 

We consider the setting of Section \ref{subsec:inf-dim}. Note that we obtain for the single machine setting with 
probability at least $1-\delta-7e^{-\frac{n}{c_2 }}$ 
\[ \eps_n \lesssim \mbe_{\beta^* ,\epsilon} \left[||\Sigma^{\frac{1}{2}}(\bar \beta_1 - \beta^*) ||^2 \right] \lesssim  
\max\left\{ \eps_n, \frac{C_{\alpha, n}}{\sqrt{n}} \right\}   \;, \]
with 
\[ C_{\alpha, n}= \frac{1}{\alpha}\eins\{ \alpha> 0\} + \frac{1}{\eps_n}\eins\{ \alpha=0 \}  \;. \]
This follows from  Proposition \ref{prop:poly} and Corollary \ref{cor:opt-poly} (in particular \eqref{eq:equiv}, \eqref{eq:lower-high}).
This bound is optimal if $\frac{C_{\alpha, n}}{\sqrt{n}} \lesssim \eps_n$. 

Similarly, denoting by $M_{opt}$ the optimal number of splits from Corollaries \ref{cor:poly1}, \ref{cor:poly2}, 
with probability at least $1-\delta-7M_{opt}e^{-\frac{n}{c_2 M_{opt}}}$
\[ \frac{\eps_n}{M_{opt}} \lesssim  \mbe_{\beta^* ,\epsilon} \left[||\Sigma^{\frac{1}{2}}(\bar \beta_{M_{opt}} - \beta^*) ||^2 \right]   \lesssim   \frac{\eps_n}{M_{opt}}\;.  \]

As a result, optimal data splitting leads to a linear increase in efficiency: 

\begin{corollary}
Let all assumptions of Corollaries \ref{cor:poly1}, \ref{cor:poly2}, \ref{cor:opt-poly} be satisfied and assume 
that $\frac{C_{\alpha, n}}{\sqrt{n}} \lesssim \eps_n$. Then, with probability at least $1-\delta-7M_{opt}e^{-\frac{n}{c_2 M_{opt}}}$ 
\[  {\rm \widehat{Eff} }(M_{opt})  \simeq  M_{opt} \;.\] 
\end{corollary}


\section{Discussion}
\label{sec:discussion}

{\bf Comparison to averaged ordinary least squares (AOLS).} To understand the regularizing effect of the number of data-splits we compare 
our approach to AOLS, i.e. \eqref{eq:ridgeless-2}, \eqref{eq:weighted-ave} in the 
\emph{underparameterized} regime. This has been studied in e.g. \cite{rosenblatt2016optimality}. 
Since OLS is \emph{unbiased}, AOLS is unbiased, too, and the risk behaves 
fundamentally different as a function of $M$. In particular, there is no trade-off between bias and variance. 
The performance in this setting for fixed $d$ as $n\to \infty$ is comparable to the single machine setting. However, 
in the high dimensional limit $d/n \to \gamma \in (0,1)$, data splitting incurs a loss in accuracy that increases linearly with $M$ and 
we trade accuracy for speed. Notably, in the overparameterized 
regime, we observe an additional bias and hence an \emph{increase} in efficiency until the optimum $M_n$ is achieved (see Fig. \ref{fig:1}), 
see Section \ref{app:efficiency} for a more extended discussion.    

{\bf Comparison to distributed Ridge Regression (DRR).} 
The learning properties of the distributed ridgeless estimator also changes with additional regularization as for (kernel) ridge regression. This setting 
is extensively investigated in kernel learning e.g. \cite{zhang2015divide}, \cite{lin2017distributed}, \cite{mucke2018parallelizing}. 
In this setup, the averaged estimator suffers \emph{no loss} in accuracy, i.e. no increase in efficiency, if appropriately regularized, 
provided the number of machines grows sufficiently slowly with the sample size. 
\\  
The work \cite{sheng2020one} investigates DRR in the high dimensional limit and finds that the efficiency is generally high when the signal strength is low. 
Note that we observe a similar phenomenon in Corollary \ref{cor:sparse-strong-weak} through the signal-to-noise-ratio $\SNR$. 
A  low $\SNR$ increases the optimal number $M_n$. 
Moreover, the authors show that even in the limit of many machines, DRR does not lose
all efficiency. We show in \eqref{eq:to-zero} that in the infinite worker limit, the risk converges to zero if $d_n$ increases.



\section{Numerical Illustration}
\label{sec:numerics}

In this section we present some numerical examples, illustrating our main findings. Additional numerical results are presented in Appendix \ref{app:add-numerics}. 
\vspace{0.1cm}
\\
{\bf Simulated data.} We illustrate the findings of Section \ref{subsec:sparse-finite-dim} in the \emph{strong-weak-features model}. 
In a first experiment we generate $n=200$ i.i.d. training points $x_j \sim \cN(0, \Sigma)$, with $d=600$, $\rho_1=1$, $\rho_2=10^{-4}$. 
The target $\beta^*$ is simulated according to Assumption \ref{ass:sparse} with $\SNR=0.1$. 
We illustrate the effect of the number $F$ of strong features on the \emph{relative efficiency} compared to the non-distributed setting, 
i.e. the ratio of the test risk of $\bar\beta_{M}$ and $\bar \beta_1$.   
The number of the strong features is $F=100, 150, 200$. The left plot in Fig. \ref{fig:1} shows that the efficiency for larger $F$ is generally higher. 
Interestingly, for fixed $F$, efficiency increases until the optimal number of splits is achieved. In other words, $M$ acts as a regularization parameter.   
As predicted by Corollary \ref{cor:sparse-strong-weak}, 
the optimal number of splits decreases as $F$ increases. 
In a second experiment, we investigate the interplay of the spectral gap $\rho_1 - \rho_2$ and the optimal splits.  
The strength $\rho_2$ of weak features varies between $10^{-3}$ and $10^{-1}$. The right plot in Fig. \ref{fig:1} plots 
the test error for different values of the spectral gap for an increasing number of machines. We clearly observe the regularizing 
effect of data splitting in the presence of overparameterization. Moreover, as predicted by Corollary \ref{cor:sparse-strong-weak}, 
the optimal number of splits decreases as the spectral gap increases. 
\vspace{0.1cm}
\\
{\bf Real data.}
We utilize the million song dataset \cite{bertinmillion}, consisting of $463, 715$ training samples, $n_{test}=51,630$ test samples and $d=90$ features. 
To illustrate the effect of splitting we elaborate two different settings: 
The left plot shows data splitting in the presence of global overparameterization. We subsampled $n=45$ training samples and report the average test error 
with $100$ repetitions. We observe a better accuracy with splitting. In the second setting, the total sample size is larger than the 
number of parameter. As long as there is local underparameterization, the test error increases.  
However, after a certain number of splits $M=n/d=15$, local overparameterization appears and the test error starts to decrease.  

\begin{figure}[ht]
\label{fig:1}
\centerline{
\includegraphics[width=0.45\columnwidth, height=0.25\textheight]{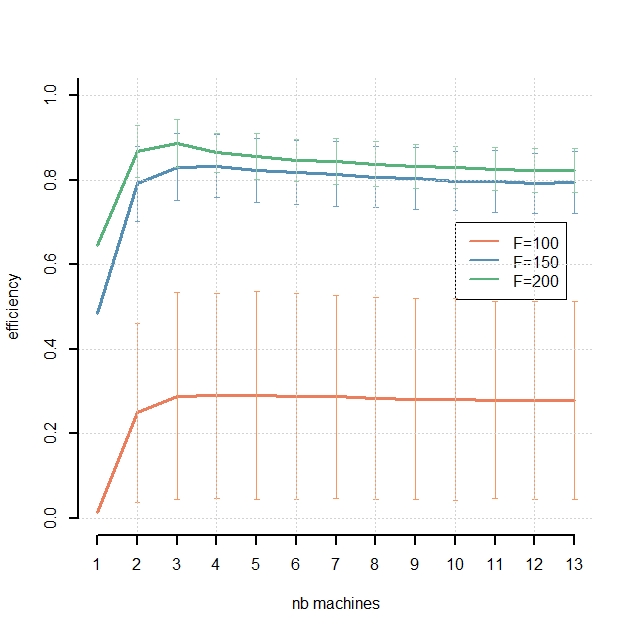}
\includegraphics[width=0.45\columnwidth, height=0.25\textheight]{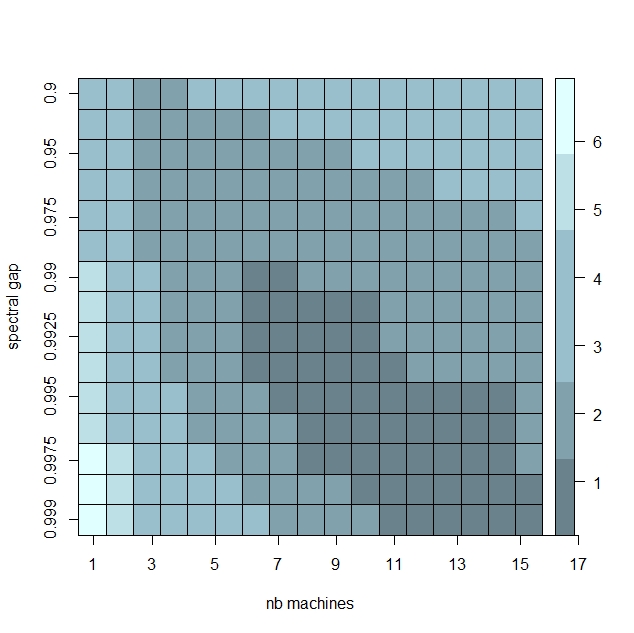}
}
\caption{{\bf Left:} Interplay between number of strong features and efficiency. 
{\bf Right:} Interplay between spectral gap and optimal number of machines.}
\end{figure}

\begin{figure}[ht]
\label{fig:2}
\centerline{
\includegraphics[width=0.45\columnwidth, height=0.25\textheight]{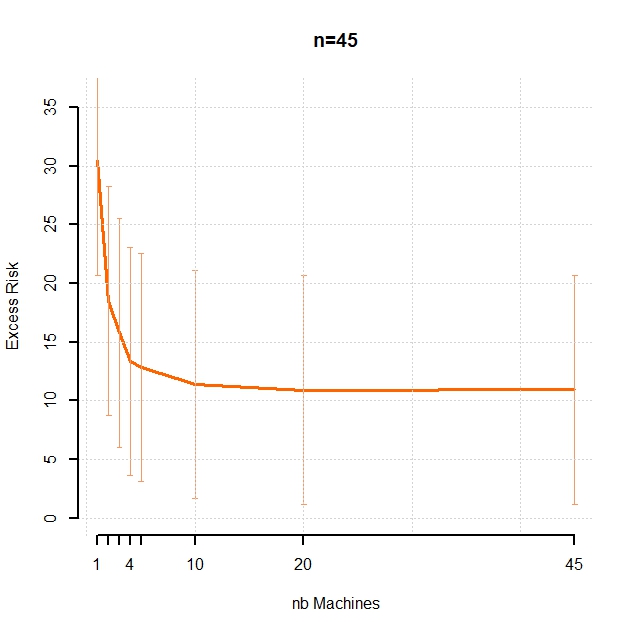} 
\includegraphics[width=0.45\columnwidth, height=0.25\textheight]{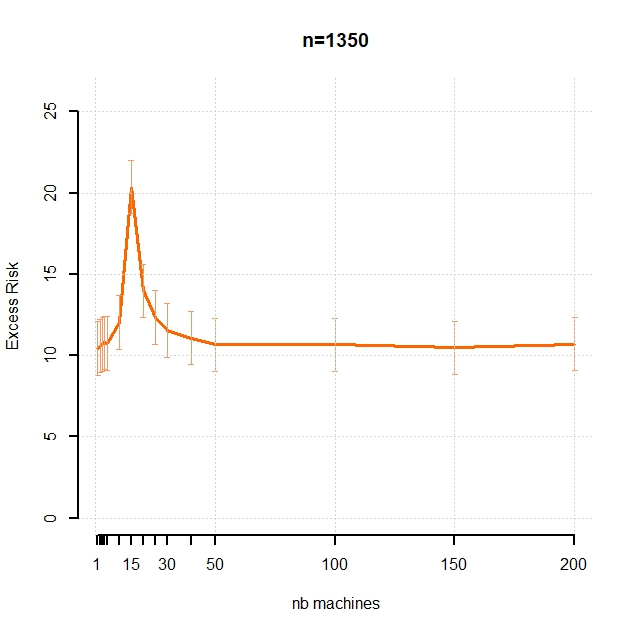}
}
\caption{ MSDYear dataset. {\bf Left:} Data splitting reduces the test error in the presence of overparameterization. 
{\bf Right:} The test error has a peak for $M=n/d$ and decreases as local overparameterization increases.}
\end{figure}



\bibliographystyle{alpha}
\bibliography{bib_distributed}


\newpage

\appendix


\section{Proofs of Section \ref{sec:main-results}}
\label{supp:proofsA}

In this section we provide all proofs of our results in Section \ref{sec:main-results}.


\subsection{Proofs of Section \ref{subsec:general-upper-bound}}
\label{app:general-bound}

\begin{lemma}
\label{lem:bias-intro}
Let $n \in \mbn$, $\beta \in \cH$. Define the \emph{empirical covariance} operator by $\hat \Sigma=\frac{1}{n}\bX^T\bX$ and denote by 
\[ \tilde \Pi := Id - \bX^T(\bX\bX^T )^{\dagger}\bX    \]
the orthogonal projection onto the nullspace of $\bX$.  We have almost surely 
\[ || \Sigma^{1/2} \tilde \Pi \beta ||^2 \leq  \left| \inner{\beta , (\Sigma - \hat \Sigma ) \beta} \right|\;. \]
\end{lemma}

\begin{proof}[Proof of Lemma \ref{lem:bias-intro}] 
For the proof we will use the following facts that can be found in e.g. \cite{reed2012methods}: 
\begin{itemize}
\item[(a)] For all $\beta \in \cH$ it holds: $||\beta||^2 = \tr[\beta \otimes \beta]$.  
\item[(b)] The trace is invariant under cyclic permutations: $\tr[ABC] = \tr[CAB] =\tr[BCA]$. 
\item[(c)] If $A,B,C$ are self-adjoint, then the trace is invariant under any permutation: 
\[ \tr[ABC] = \tr[(ABC)^T] = \tr[CBA] = \tr[ACB] \;. \]
\item[(d)] If $A$ has rank one, then $|\tr[A]| = ||A||$. In particular, $\beta \otimes \beta$ has rank one and 
$|\tr[\beta\otimes \beta A ] |= |\tr[ A \beta\otimes \beta ]| = ||\beta\otimes \beta A|| $.  
\end{itemize}
First observe that 
\begin{align*}
 || \Sigma^{1/2} \tilde \Pi \beta^*||^2 &\stackrel{(a)}{=} \left| \tr\left[   \Sigma^{1/2} \tilde \Pi \beta \otimes  \Sigma^{1/2} \tilde \Pi \beta \right] \right| \\
 &= \left| \tr\left[   \Sigma^{1/2} \tilde \Pi ( \beta \otimes  \beta )  \tilde \Pi \Sigma^{1/2} \right] \right| \\
 &\stackrel{(b)}{=} \left|  \tr\left[ \tilde \Pi  \Sigma\tilde \Pi ( \beta \otimes  \beta )    \right] \right| \\
&\stackrel{(d)}{=}  || \tilde \Pi  \Sigma\tilde \Pi ( \beta \otimes  \beta )  || =: \bullet  \;.
\end{align*}
Since $\tilde \Pi$ is an orthogonal projection onto the nullspace of $\bX$ we have $||\tilde \Pi|| \leq 1$ and  
\[ \tilde \Pi \bX^T =0 \;, \quad \tilde \Pi \hat \Sigma = 0 \;.  \]
Hence, we find 
\begin{align*}
 \bullet &=  || \tilde \Pi  (\Sigma - \hat \Sigma ) \tilde \Pi ( \beta \otimes  \beta )  || \\
&\leq ||\tilde \Pi|| \; ||  (\Sigma - \hat \Sigma ) \tilde \Pi ( \beta \otimes  \beta )  || \\
&\stackrel{(d)}{=} |\tr[ (\Sigma - \hat \Sigma ) \tilde \Pi ( \beta \otimes  \beta ) ] | \\
&\stackrel{(c)}{=}  |\tr[ \tilde \Pi  (\Sigma - \hat \Sigma )  ( \beta \otimes  \beta ) ] |\\
&\stackrel{(d)}{=} || \tilde \Pi  (\Sigma - \hat \Sigma )  ( \beta \otimes  \beta ) ||  \\
&\leq   ||  (\Sigma - \hat \Sigma )  ( \beta \otimes  \beta ) ||   \\
&\stackrel{(d)}{=}   |  \tr[ (\Sigma - \hat \Sigma )  ( \beta \otimes  \beta ) ] | \\
&= \left| \inner{\beta , (\Sigma - \hat \Sigma ) \beta} \right|\;.
\end{align*}
\end{proof}

\vspace{0.1cm}

The next Proposition is useful for bounding the bias in Lemma \ref{lem:bias-var}. We follow the lines of \cite{negrea2020defense}, Lemma B.1,  
where a similar result is shown for gaussian variables.  
We extend this to the subgaussian setting. 

\vspace{0.1cm}

\begin{proposition}
\label{prop:bias-high-proba}
Suppose Assumption \ref{ass:design} is satisfied and let $\beta \in \cH$. There exists a universal constant $c>0$ such that for any 
$\delta \geq 2e^{-c^2n}$, with probability at least $1-\delta$ we have  
\[  \left| \inner{\beta , (\Sigma - \hat \Sigma ) \beta} \right| 
\leq \frac{4 \sigma_x}{c} \; \log^{\frac{1}{2}}(2/\delta) \;\frac{||\Sigma^{\frac{1}{2}}\beta||^2}{\sqrt n} \;. \]
\end{proposition}

\begin{proof}[Proof of Proposition \ref{prop:bias-high-proba}]
Set $B^2:=\inner{\Sigma \beta , \beta}$. We then write 
\begin{align}
\label{eq:quantity}
 \left| \inner{\beta , (\Sigma - \hat \Sigma ) \beta} \right| 
&= \left|  \inner{\beta , \hat \Sigma  \beta} -  \inner{\beta , \Sigma   \beta}  \right| \nonumber \\
&= \left|  \frac{1}{n}\sum_{j=1}^n \inner{\beta , (x_j \otimes x_j)\beta}  - B^2 \right| \nonumber \\
&=  \left|  \frac{B^2}{n} \left( \sum_{j=1}^n \frac{\inner{\beta, x_j}^2}{B^2} - 1  \right) \right| \;.
\end{align}
We next show that for any $j=1,...,n$, the real valued variables $z_j:= \frac{\inner{\beta, x_j}}{B}$ are $(\sigma_x^2, id)$-subgaussian. 
Indeed, by Assumption \ref{ass:design} and Definition \ref{def:subgaussian} we find for all $\alpha \in \mbr$  
\begin{align}
\label{eq:zj-is-subgauss}
\mbe\left[ e^{\alpha z_j}\right] &= \mbe\left[ e^{ \inner{ \frac{\alpha}{B} \beta, x_j } }\right] \nonumber \\
&\leq e^{\frac{\sigma_x^2}{2} \inner{\Sigma \frac{\alpha}{B}\beta , \frac{\alpha}{B}\beta } } \nonumber \\
&= e^{\frac{\sigma_x^2}{2}\alpha^2} \;.
\end{align}
For bounding \eqref{eq:quantity} with high probability we use the fact that for any $j=1,...,n$ the random variable $z_j^2-1$ is $16\sigma_x^2$-subexponential. 
Indeed, this follows from \eqref{eq:zj-is-subgauss} and results in \cite[Section 2]{vershynin2018high} that are condensed in \cite[Lemma S.4]{bartlett2020benign}. 
Next, Bernstein's inequality for the independent and mean zero subexponential variables $z_1, ..., z_n$ in \cite[Theorem 2.8.2]{vershynin2018high} 
shows that there exists a universal constant $c>0$ such that for all $t\geq 0$,  with probability at least 
\[ 1-2\exp\left( -c^2 \min\left\{ \frac{t^2}{16 \sigma_x^2 n} , \frac{t}{4 \sigma_x}\right\} \right) \]
we have 
\[ \left|  \frac{B^2}{n} \left( \sum_{j=1}^n \frac{\inner{\beta, x_j}^2}{B^2} - 1  \right) \right|  \leq  \frac{B^2}{n} \;t \;. \] 
Assuming that $t \leq 4\sigma_x n$ we find that 
\[  \min\left\{ \frac{t^2}{16 \sigma_x^2 n} , \frac{t}{4 \sigma_x}\right\}  = \frac{t^2}{16 \sigma_x^2 n} \;. \]
Setting now $\delta = 2e^{-\frac{c^2}{16\sigma_x^2}\frac{t^2}{n}}$ we finally get 
\[  \left| \inner{\beta , (\Sigma - \hat \Sigma ) \beta} \right|  \leq \frac{4\sigma_x}{c} \frac{||\Sigma^{\frac{1}{2}}\beta ||^2}{\sqrt n} \log^{\frac{1}{2}}(2/\delta) \;. \]
with probability at least $1-\delta$, for all $\delta \geq 2e^{-c^2n}$.
\end{proof}

\vspace{0.2cm}

The next result establishes a bound for the single machine variance. This is a first step for bounding the variance from Lemma \ref{lem:bias-var} in the distributed setting. 

\vspace{0.1cm}

\begin{proposition}
\label{prop:var-intro}
Let $n \in \mbn$ and suppose Assumption \ref{ass:design} is satisfied. 
Define 
\[ \widehat{C}:= (\bX\bX^T)^{-1}\bX \Sigma \bX^T(\bX\bX^T)^{-1} \;.  \]
There exists a universal constant $c>1$ and a $0 \leq k_n^* \leq \frac{n}{c}$  
such that with probability at least $1-7e^{-\frac{n}{c}}$ it holds 
\[  \tr[\widehat{C}] \leq c \left( \frac{k_n^*}{n} + n\;\frac{\sum_{j>k_n^*}  \lam_j^2 }{ \left(\sum_{j>k_n^*}  \lam_j \right)^2 }\right) \;,  \]
where $(\lam_j)_{j \in \mbn}$ are the eigenvalues of $\Sigma$, arranged in decreasing order. 
\end{proposition}

\begin{proof}[Proof of Proposition \ref{prop:var-intro}]
The proof follows from \cite[Lemma 6]{bartlett2020benign} and \cite[Lemma 11]{bartlett2020benign}.
\end{proof}


\vspace{0.2cm}

Combining now the above results allows to prove Lemma \ref{lem:bias-var}. 

\vspace{0.1cm}

\begin{proof}[Proof of Lemma \ref{lem:bias-var}]
We first derive a bound for the bias. Linearity of the expectation and \eqref{eq:model} yields 
\begin{equation}
\label{eq:mbe-eps}
 \mbe_\eps[ \bar \beta_{M}] = \frac{1}{M}\sum_{m=1}^M X_m^T(X_mX_m^T )^{\dagger}\mbe_\eps[ Y_m] = 
\frac{1}{M} \sum_{m=1}^M X_m^T(X_mX_m^T )^{\dagger}X_m\beta^* \;, 
\end{equation}
since, conditionally on the inputs $\bX$, the noise is centered. Hence 
\begin{align*}
\beta^* - \mbe_\epsilon[\bar \beta_{M}]  &= \frac{1}{M}\sum_{m=1}^M  \tilde \Pi_m \beta^* \;,
\end{align*}
where we denote by 
\[ \tilde \Pi_m := Id - \bX_m^T(\bX_m\bX_m^T )^{\dagger}\bX_m    \]
the orthogonal projection onto the nullspace of $\bX_m$.  
Convexity and Lemma \ref{lem:bias-intro} allow to deduce  
\begin{align*}
\bias(\bar \beta_{M}) &= ||\Sigma^{1/2}(\mbe_\epsilon[\bar \beta_{M}] - \beta^*)||^2 \\
& \leq \frac{1}{M}  \sum_{m=1}^M || \Sigma^{1/2}  \tilde \Pi_m \beta^*||^2 \\ 
&\leq  \frac{1}{M}  \sum_{m=1}^M  \left| \inner{\beta^* , (\Sigma - \hat \Sigma_m ) \beta^*} \right| \;.
\end{align*}
Next, we derive a bound for the variance. 
By definition of the variance, \eqref{eq:ridgeless-2} and \eqref{eq:mbe-eps} we find  
\begin{align*}
\var (\bar \beta_{M})  &= \mbe_\epsilon \left[||\Sigma^{1/2} (\bar \beta_{M} -  \mbe_\epsilon[\bar \beta_{M} ] ) ||_2^2  \right] \\
&= \mbe_\eps\Big[ || \Sigma^{1/2} \Big( \frac{1}{M}  \sum_{m=1}^M   \hat \beta^{(m)} -  X_m^T(X_mX_m^T )^{\dagger}X_m\beta^*\Big) ||^2 \Big] \\
&= \mbe_\eps\Big[ || \Sigma^{1/2} \Big( \frac{1}{M}  \sum_{m=1}^M   X_m^T (X_m X_m^T )^{\dagger}  (Y_m - X_m\beta^*) \Big) ||^2 \Big] \\
&= \mbe_\eps\Big[ || \Sigma^{1/2} \Big( \frac{1}{M}  \sum_{m=1}^M   X_m^T (X_m X_m^T )^{\dagger} \eps_m \Big) ||^2 \Big] \\
&= \mbe_\eps\Big[ || \Sigma^{1/2} \Big( \frac{1}{M}  \sum_{m=1}^M  X_m^\dagger \eps_m \Big) ||^2 \Big] \;.
\end{align*} 
In the last step we use $ X_m^T (X_m X_m^T )^{\dagger} = X_m^\dagger $. 
Recall that for any $\beta \in \cH$ we may write $||\beta||^2 = \tr[\beta \otimes \beta]$. Hence, 
\begin{align*}
 &|| \Sigma^{1/2} \Big( \frac{1}{M}  \sum_{m=1}^M  X_m^\dagger \eps_m \Big) ||^2 \\
&\tr\left[ \left( \frac{1}{M}  \sum_{m=1}^M \Sigma^{1/2}  X_m^\dagger \eps_m  \right) \otimes 
 \left( \frac{1}{M}  \sum_{m'=1}^M  \Sigma^{1/2}  X_{m'}^\dagger \eps_{m'}  \right)\right] \\
 &= \frac{1}{M^2}  \sum_{m, m' =1}^M   \tr\left[ \Sigma^{1/2}  X_m^\dagger
  \eps_m \otimes \eps_{m'} (X_{m'}^\dagger)^T \Sigma^{1/2} \right] \;.
\end{align*}
By linearity of the trace and independence, taking the expectation gives $\mbe_\eps[\eps_m \otimes \eps_{m'} ] = 0$ for any $m \not = m'$ and the sum reduces to 
\begin{align}
\label{eq:tr1}
 \mbe_\eps\Big[ || \Sigma^{1/2} \Big( \frac{1}{M}  \sum_{m=1}^M  X_m^\dagger \eps_m \Big) ||^2 \Big]  
&= \frac{1}{M^2}  \sum_{m=1}^M \mbe_\eps\Big[  \tr\Big[  \Sigma^{1/2}  X_m^\dagger \eps_m \otimes \eps_{m} (X_{m}^\dagger)^T \Sigma^{1/2} \Big]  \Big] \nonumber  \\
&= \frac{1}{M^2}  \sum_{m=1}^M \mbe_\eps\Big[ || \Sigma^{1/2}  X_m^\dagger \eps_m||^2  \Big] \nonumber  \\
&= \frac{1}{M^2}  \sum_{m=1}^M \mbe_\eps\Big[  \inner{\eps_m , C_m \eps_m }  \Big]\;,
\end{align} 
where we set 
\[  C_m := \left( X_m^\dagger \right)^T \Sigma X_m^\dagger \;. \]
To proceed, we apply a conditional subgaussian version of the Hanson-Wright inequality taken from \cite[Lemma 35]{page2019ivanov}. This gives almost surely 
conditional on the data $X_m$, for all $t\geq 0$, with probability at least $1-e^{-t}$ (w.r.t. the noise) 
\begin{align*}
\inner{\eps_m , C_m \eps_m } &\leq \tau^2 \tr[C_m] + 2\tau^2t||C_m|| + 2\tau^2\sqrt{t^2 ||C_m||^2 + t\tr[C_m^2]}   \\
&\leq 4\tau^2\tr[C_m]\;(t+1) \;,
\end{align*} 
where we use that $||C_m|| \leq \tr[C_m]$ and $\tr[C_m^2]\leq ||C_m||\tr[C_m] \leq \tr[C_m]^2$. 
From \cite[Lemma C.1]{blanchard2018optimal} we obtain after integration 
for the conditional expectation 
\[ \mbe_\eps\left[ \inner{\eps_m , C_m \eps_m } \right] \leq 8 \tau^2\tr[C_m] \;.\]
Inserting the last bound into \eqref{eq:tr1} finally gives almost surely 
\[ \var (\bar \beta_{M}) \leq \frac{8 \tau^2}{M^2}  \sum_{m=1}^M \tr[C_m] \;. \]
\end{proof}


\vspace{0.3cm}

Finally, we give the proof of our main result, a general upper bound for distributed ridgeless regression. 

\vspace{0.1cm}

\begin{proof}[Proof of Theorem \ref{theo:main1}]
We start with bounding the bias term. 
\vspace{0.2cm}
\\
\noindent
{\bf Bounding the Bias.} Recall that by Lemma \ref{lem:bias-var} we have almost surely 
\[ \bias(\bar \beta_{M}) \leq  \frac{1}{M}  \sum_{m=1}^M  \left| \inner{\beta^* , (\Sigma - \hat \Sigma_m ) \beta^*} \right| \;. \]
Proposition \ref{prop:bias-high-proba} gives for all $\delta \geq 2e^{-c^2\frac{n}{M}}$, with probability at least $1-\delta$
\[  \left| \inner{\beta^* , (\Sigma - \hat \Sigma_m ) \beta^*} \right| 
\leq  \frac{4 \sigma_x}{c} \; \log^{\frac{1}{2}}(2/\delta) \;||\Sigma^{\frac{1}{2}}\beta^*||^2  \sqrt{\frac{M}{n}} \;, \]
for some universal constant $c>0$. Performing now a union bound and invoking Assumption \ref{ass:source} finally gives with probability 
at least $1-\delta$
\[ \mbe_{\beta^*}[\bias(\bar \beta_{M})] 
\leq  \frac{4 \sigma_x}{c} \; \log^{\frac{1}{2}}(2M/\delta) \;\tr[ \Sigma \Theta ]   \sqrt{\frac{M}{n}} \;,  \]
where we use that 
\[ \mbe_{\beta^*}[ ||\Sigma^{\frac{1}{2}}\beta^*||^2  ]  = \tr\left[ \mbe_{\beta^*}[ \Sigma \beta^*\otimes \beta^*  ] \right] = \tr[ \Sigma \Theta] \;. \]
\vspace{0.2cm}
\\
\noindent
{\bf Bounding the Variance.} Applying Lemma \ref{lem:bias-var} once more we have almost surely 
\[ \var (\bar \beta_{M}) \leq  \frac{8 \tau^2}{M^2}  \sum_{m=1}^M \tr[C_m] \;. \]
With Lemma \ref{prop:var-intro} together with a union bound we get with probability at least $1-7Me^{-\frac{n}{c_2 M}}$
\begin{align*}
\var (\bar \beta_{M}) 
&\leq   \frac{8c_2 \tau^2}{M^2} \sum_{m=1}^M  \left( \frac{M}{n}k_{n/M}^* + \frac{n}{M}\;\frac{\sum_{j>k_{n/M}^*}  \lam_j^2 }{ \left(\sum_{j>k_{n/M}^*}  \lam_j \right)^2 }  \right) \\
&= 8c_2 \tau^2  \left( \frac{k_{n/M}^*}{n} + 
 \frac{n}{M^2}\;\frac{\sum_{j>k_{n/M}^*}  \lam_j^2 }{ \left(\sum_{j>k_{n/M}^*}  \lam_j \right)^2 }  \right) \;,
\end{align*}
for some constant $c_2>1$ and $0 \leq k_{n/M}^* \leq \frac{n}{c_2M}$. 
\end{proof}

\[\]


\subsection{Proofs of Section \ref{subsec:inf-dim}}
\label{app:special-cases}

This section establishes a refined upper bound for the excess risk in the infinite dimensional setting under the 
specific Assumptions \ref{ass:source-hoelder}, \ref{ass:poly-decay}. We start with a preliminary Lemma that is needed to estimate the variance.  

\begin{lemma}
\label{lem:poly}
Suppose all assumptions of Theorem \ref{theo:main1} are satisfied. Assume that $\lam_j(\Sigma)=j^{-(1+\eps_n)}$ for a positive sequence $(\eps_n)_{n \in \mbn}$.  
We have  
\begin{enumerate}
\item $ \frac{k_{\frac{n}{M}}^*}{n} \leq a\;\frac{\eps_n}{M}$.  
\item 
For any $n$ sufficiently large, 
$\frac{n}{M^2}\;\frac{1}{R_{k_{\frac{n}{M}}^*}(\Sigma)} \leq \frac{6}{a}\; \frac{\eps_n}{M}$. If $M \lesssim n \eps_n$, then $k^*_{\frac{n}{M}}\gtrsim 1$.
\end{enumerate}
\end{lemma}

\begin{proof}[Proof of Lemma \ref{lem:poly}]
\begin{enumerate}
\item We follow the lines of \cite{bartlett2020benign}, Proof of Theorem 31, by lower bounding the effective rank $r_k(\Sigma)$. 
With Lemma 14 in \cite{mucke2019beating} we may write 
\begin{align*}
r_k(\Sigma) & = (k+1)^{1+\eps_n} \sum_{j>k} j^{-(1+\eps_n)} \\
&\geq (k+1)^{1+\eps_n} \int_{k+1}^\infty t^{-(1+\eps_n)} \\
&= \frac{k+1}{\eps_n} \;.
\end{align*}
By Definition \ref{def:eff-dim}, the effective dimension $k^*_{\frac{n}{M}}$ is the smallest number satisfying 
$r_{k^*_{\frac{n}{M}}}(\Sigma) \geq a \frac{n}{M}$. Hence, $k^*_{\frac{n}{M}} \leq a\eps_n\frac{n}{M}$.

\item A short calculation shows that 
\[ R_k(\Sigma) \geq \frac{k}{\eps_n^2}\left( 1- \frac{1}{k+1} \right)^{2\eps_n} \;, \quad r_k(\Sigma) \leq \frac{2k}{\eps_n}e^{\eps_n} \;. \]
Following the arguments in the proof of Theorem 31 in \cite{bartlett2020benign} we find also in the distributed setting 
that $k^*_{\frac{n}{M}}\geq \frac{a \eps}{3} \frac{n}{M}$ for sufficiently large $n$. Hence, 
\[  R_{k^*_{\frac{n}{M}}} (\Sigma) \geq   \frac{a \eps_n }{6} \frac{n}{M} \;. \] 
\end{enumerate}
\end{proof}

\vspace{0.1cm}

The second preliminary Lemma will help to bound the bias. 

\vspace{0.1cm}

\begin{lemma}
\label{lem:source}
Suppose Assumption \ref{ass:poly-decay} is satisfied. 
Let $\alpha \geq 0$. Then 
\[ \tr[ \Sigma^{1+\alpha} ]  \; = \; \sum_{j=1}^\infty \left( \frac{1}{j} \right)^{(1+\alpha )(1+\eps_n)}  \; \leq \;  \frac{1}{\alpha + \eps_n(1+\alpha)}\;  \leq \; 
  \left\{ {\begin{array}{cc}
    \alpha=0 &:  \; \frac{1}{\eps_n} \\
    \alpha >0 &: \;  \frac{1}{\alpha} \\
  \end{array} } \right. 
 \;. \]
\end{lemma}

\vspace{0.1cm}

\begin{proof}[Proof of Lemma \ref{lem:source}]
Let $\beta=(1+\alpha )(1+\eps_n)$. 
The infinite sum can easily be bounded by an integral 
\begin{align*}
\sum_{j=1}^\infty \left( \frac{1}{j} \right)^{\beta} & \leq \int_{1}^\infty t^{-\beta} dt = \frac{1}{\beta-1} \;, 
\end{align*}
see e.g. Lemma 14 in \cite{mucke2019beating}.
\end{proof}

\vspace{0.1cm}

Combining now Lemma \ref{lem:poly} and Lemma \ref{lem:source} with  Theorem \ref{theo:main1} gives the main result in this section.

\vspace{0.1cm}

\begin{proof}[Proof of Proposition \ref{prop:poly}]
For bounding the bias we apply  Lemma \ref{lem:source} 
\[  \tr[ \Sigma \Theta]  = \tr[ \Sigma^{1+\alpha} ]  = \sum_{j \in \mbn}\lam^{1+\alpha}_j(\Sigma) 
\leq \frac{1}{\alpha}\eins\{ \alpha> 0\} + \frac{1}{\eps_n}\eins\{ \alpha=0 \}  = : C_{\alpha, n} \;.\]
Combining this  with Lemma \ref{lem:poly}, Lemma \ref{lem:bias-var} and Theorem \ref{theo:main1} leads to 
\begin{align*}
&\mbe_{\beta^* ,\epsilon} \left[||\Sigma^{\frac{1}{2}}(\bar \beta_M - \beta^*) ||^2 \right] \\
&\leq \frac{4 \sigma_x}{c_1} \log^{\frac{1}{2}}\left(\frac{2M}{\delta}\right)\tr[ \Sigma^{1+\alpha} ]   \sqrt{\frac{M}{n}}   + 
 8c_2 \tau^2 \left( a\;\frac{\eps_n}{M} +   \frac{6}{a}\; \frac{\eps_n}{M} \right) \\ 
&\leq \frac{4 \sigma_x}{c_1} C_{\alpha, n}  \log^{\frac{1}{2}}\left(\frac{2M}{\delta}\right) \sqrt{\frac{M}{n}}  + 
  8c_2c_a \tau^2 \frac{\eps_n}{M}  \;,
\end{align*}
holding with probability at least $1-\delta-7Me^{-\frac{n}{c_2 M}}$, for any $\delta \geq 2e^{-c_1^2\frac{n}{M}}$. 
Here, we set $c_a=\max\{ a, 6/a\}$.  
The result follows with $c_3=4/c_1$ and $c_4= 8c_2c_a$. 
\end{proof}

\vspace{0.1cm}

\begin{proof}[Proof of Corollary \ref{cor:poly1} and Corollary \ref{cor:poly2}]
We determine the maximum number of local nodes by balancing bias and variance. To this end, firstly note that 
$1\leq \log^{\frac{1}{2}}\left(\frac{4M_n}{\delta}\right)$.
 Setting now 
\[ A := c_3 \sigma_x  \frac{C_{\alpha, n}}{\sqrt{n}}  \;, \quad B := c_4\tau^2 \eps_n  \;, \] 
we find that 
\[ A \sqrt{M} = \frac{B}{M}  \quad \Longleftrightarrow \quad  M = \left( \frac{B}{A} \right)^{2/3}\;.  \]
Hence, the value 
\[ M_n := C_{\tau, \sigma_x}\left( \frac{\eps_n \sqrt n}{C_{\alpha ,n}}\right)^{2/3}\;, \quad  C_{\tau, \sigma_x} = \left( \frac{c_4\tau^2}{c_3 \sigma_x} \right)^{2/3}  \]
trades off bias and variance and the excess risk is bounded as 
\begin{align*}
\mbe_{\beta^* ,\epsilon} \left[||\Sigma^{\frac{1}{2}}(\bar \beta_{M_{n}} - \beta^*) ||^2 \right] 
&\leq  2c_4\tau^2 \log^{\frac{1}{2}}\left(\frac{4M_n}{\delta}\right) \frac{\eps_n}{M_n} \\
&= C'_{\tau, \sigma_x}\log^{\frac{1}{2}}\left(\frac{4M_n}{\delta}\right)  \left( \frac{C^2_{\alpha, n}\eps_n}{n} \right)^{1/3}\;, 
\end{align*}
where $C'_{\tau, \sigma_x} = \frac{2c_4\tau^2}{C_{\tau, \sigma_x}}$.
\end{proof}

\[\]


\subsection{Proofs of Section \ref{subsec:sparse-finite-dim}}
\label{sec:finite-dim-sparsity}


In this section we provide the proofs for our results in finite dimension with $dim(\cH)=d<\infty$ from Section \ref{subsec:sparse-finite-dim}. 
We start with two preliminary Lemmata.

\begin{lemma}
\label{lem:bias-random-effects}
Suppose Assumption \ref{ass:sparse} holds. Then 
\[ \tr[\Sigma \Theta] \leq  \frac{2\cdot \SNR}{d}\; F \;. \]
\end{lemma}

\vspace{0.1cm}

\begin{proof}[Proof of Lemma \ref{lem:bias-random-effects}]
By Assumption \ref{ass:sparse} with $\Theta = \frac{\SNR}{d}\;  Id_d$ and since $d \rho_2 \leq F$ we easily obtain 
\begin{align*}
\tr[\Sigma \Theta] &=   \frac{\SNR}{d}\;\tr[\Sigma ] \\
&= \frac{\SNR}{d}\; \left(  \sum_{j=1}^F \rho_1 + \sum_{j=F+1}^d \rho_2\right)  \\
&=\frac{\SNR}{d}\; \left( F \rho_1 + (d-F)\rho_2 \right) \\
&=   \frac{\SNR}{d}\; \left( ( \rho_1 - \rho_2)F + d\rho_2 \right) \\
&\leq  \frac{\SNR}{d}\;(\rho_1 - \rho_2 + 1)F \\
&\leq \frac{2 \SNR}{d}\; F \;.
\end{align*}
In the last step we use that $\rho_1 = 1$ and $-\rho_2 \leq 0$. 
\end{proof}

\vspace{0.2cm}

\begin{lemma}
\label{lem:var-strong-wak}
Suppose Assumptions of Theorem \ref{theo:main1} are satisfied. If additionally Assumption \ref{ass:stron-weak} holds, then 
with probability at least $1-7Me^{-\frac{n}{c_2 M}}$ 
\begin{align*}
 \var (\bar \beta_M) &\leq  16c_2 \tau^2 \;  \frac{ 1}{(\rho_1-\rho_2)^2} \; \frac{n}{M^2}\;\frac{1}{F} \;. 
\end{align*} 
\end{lemma}

\vspace{0.1cm}

\begin{proof}[Proof of Lemma \ref{lem:var-strong-wak}]
We bound the first term $\frac{k_{\frac{n}{M}}^*}{n}$ in the variance. To this end, let $k < F$ and consider 
\begin{align}
\label{eq:eff-rank2}
r_k(\Sigma) &= \frac{1}{\lam_{k+1}} \sum_{j>k}\lam_j (\Sigma) \no \\
&=  \frac{1}{\lam_{k+1}} \left(  \sum_{j=k+1}^F\lam_j (\Sigma) +  \sum_{j=F+1}^d\lam_j (\Sigma)  \right) \no  \\
&=  \frac{1}{\rho_1} \left(  \rho_1(F-k) + \rho_2(d-F) \right) \no  \\
&= \frac{1}{\rho_1} \left( (\rho_1-\rho_2)F + \rho_2 d -\rho_1k \right) \no  \\
&= \left(1-\frac{\rho_2}{\rho_1} \right)F + \frac{\rho_2}{\rho_1}d - k \;.
\end{align}
Thus, 
\[  r_k(\Sigma) \geq a \frac{n}{M} \quad \Longleftrightarrow \quad k \leq  \left(1-\frac{\rho_2}{\rho_1} \right)F + \frac{\rho_2}{\rho_1}d -  a \frac{n}{M} \;. \]
The right hand side is non-negative  if we require
\begin{equation}
\label{eq:best-k} 
  a \frac{n}{M} <  \left(1-\frac{\rho_2}{\rho_1} \right)F + \frac{\rho_2}{\rho_1}d \;. 
\end{equation}  
By the definition \ref{def:eff-dim} of the effective dimension and from \eqref{eq:eff-rank2} we obtain 
\begin{align}
\label{eq:eff-dim2}
k^*_{\frac{n}{M}} &= \min\left\{ k\geq 0\;:\; r_k(\Sigma) \geq a \frac{n}{M} \right\} \no   \\
&= \min\left\{ k\geq 0\;:\; k \leq  \left(1-\frac{\rho_2}{\rho_1} \right)F + \frac{\rho_2}{\rho_1}d -  a \frac{n}{M} \right\} \no   \\
&=0 \;,
\end{align}
provided \eqref{eq:best-k} is satisfied. 

Next, we derive an upper bound for the second term in the variance. Recall that by Definition \ref{def:eff-ranks} 
\[ \frac{1}{R_{k^*}(\Sigma)} =  \frac{1}{R_{0}(\Sigma)} =  \frac{\sum_{j=1}^d  \lam_j^2(\Sigma) }{ \left(\sum_{j=1}^d  \lam_j(\Sigma) \right)^2 } \;. \]
By Assumption  \ref{ass:stron-weak} and for $\rho_2^2 d \leq F$ we may write 
\begin{align}
\label{eq:squared-EV}
\sum_{j=1}^d  \lam_j^2(\Sigma)  &= \sum_{j=1}^F \lam_j^2(\Sigma) + \sum_{j=F+1}^d \lam_j^2(\Sigma) \no \\
&= \theta^4\left( \rho_1^2 F + \rho_2^2(d-F)  \right) \no   \\ 
&= \theta^4\left( (\rho_1^2-\rho_2^2)F + \rho_2^2 d \right) \\
&\leq  \theta^4\left( 1+ \rho_1^2-\rho_2^2 \right) F  \no  \\
&\leq \theta^4\left( 1+ \rho_1^2 \right) F  \;. \no 
\end{align}
Moreover, since $\rho_1 > \rho_2 > 0$ we obtain 
\begin{align}
\label{eq:EV-not-squared}
\sum_{j=1}^d  \lam_j(\Sigma) &= \sum_{j=1}^F \lam_j(\Sigma) + \sum_{j=F+1}^d \lam_j(\Sigma) \no  \\
&=\theta^2\left( (\rho_1-\rho_2 )F + \rho_2 d \right) \\
&\geq \theta^2 (\rho_1-\rho_2 )F \;. \no 
\end{align}
Hence, 
\begin{equation}
\label{eq:var-sec}
\frac{n}{M^2}\;\frac{1}{R_{k_{\frac{n}{M}}^*}(\Sigma)} 
\leq \frac{n}{M^2}\; \frac{\theta^4\left( 1+ \rho_1^2 \right) F}{\theta^4 (\rho_1-\rho_2 )^2F^2 } 
=  \frac{ 1+ \rho_1^2}{(\rho_1-\rho_2)^2} \; \frac{n}{M^2}\;\frac{1}{F} \;.
\end{equation} 
Combining now \eqref{eq:eff-dim2} and \eqref{eq:var-sec} finally gives with probability at least $1-7Me^{-\frac{n}{c_2 M}}$ 
\begin{align*}
 \var (\bar \beta_M) &\leq  8c_2 \tau^2 \;  \frac{ 1+ \rho_1^2}{(\rho_1-\rho_2)^2} \; \frac{n}{M^2}\;\frac{1}{F} \\
 &=  16c_2 \tau^2 \;  \frac{ 1}{(\rho_1-\rho_2)^2} \; \frac{n}{M^2}\;\frac{1}{F} \;,
\end{align*} 
where in the last step we use $ 1+ \rho_1^2 = 2$. 
\end{proof}

\vspace{0.2cm}

\begin{proposition}[Restatement of Proposition \ref{prop:sparse-strong-weak}]
\label{prop:sparse-strong-weak-2}
In addition to all assumptions of Theorem \ref{theo:main1}, suppose 
Assumptions \ref{ass:sparse}, \ref{ass:stron-weak} hold.  Assume that the weak features satisfy $\rho_2 d \leq F$ and that 
\begin{equation}
\label{eq:cond-F-2}
 a \frac{n}{M} <  \left(1-\rho_2 \right) F + \rho_2 d \;.
\end{equation} 
If $ \frac{n}{M} \geq \frac{1}{c_1}\log(2/\delta)$,  
then with probability at least $1-\delta-7Me^{-\frac{n}{c_2 M}}$, the excess risk satisfies 
for some $c, \tilde c>0$
\begin{align}
\label{eq:rhs-2}
 \hspace{-0.8cm}  \mbe_{\beta^* ,\epsilon} \left[||\Sigma^{\frac{1}{2}}(\bar \beta_M - \beta^*) ||^2 \right]  
&\leq   \;  c\sigma_x \frac{ \SNR }{d}\;F  \log^{\frac{1}{2}}\left(\frac{2M}{\delta}\right)   \sqrt{\frac{M}{n}} 
 +   \tilde c \tau^2 C_{\rho_2} \; \frac{n}{M^2}\;\frac{1}{F}  \;, 
\end{align}
where we set $C_{\rho_2}=\frac{ 1}{( 1-\rho_2)^2}$. 
\end{proposition}

\vspace{0.1cm}

\begin{proof}[Proof of Proposition \ref{prop:sparse-strong-weak}] 
The proof follows directly from Lemma \ref{lem:bias-random-effects}, Lemma \ref{lem:var-strong-wak} and Theorem \ref{theo:main1}. 
Hence, if $\log(2/\delta)\leq c_1^2\frac{n}{M}$, we find with probability at least $1-\delta-7Me^{-\frac{n}{c_2 M}}$ 
\begin{align*}
\mbe_{\beta^* ,\epsilon} \left[||\Sigma^{\frac{1}{2}}(\bar \beta_M - \beta^*) ||^2 \right] 
&\leq\frac{4 \sigma_x}{c_1} \log^{\frac{1}{2}}\left(\frac{2M}{\delta}\right) \frac{2\cdot \SNR}{d}\; F  \sqrt{\frac{M}{n}}  + 
 16c_2 \tau^2   \frac{ 1}{(\rho_1-\rho_2)^2} \; \frac{n}{M^2}\;\frac{1}{F}  \;.
\end{align*}
Setting $c=8/c_1 $ and $\tilde c=16c_2$ proves our result.   
\end{proof}

\vspace{0.2cm}

\begin{proof}[Proof of Corollary \ref{cor:sparse-strong-weak}]
We need to determine the minimum of the function $h: \mbr_+ \to \mbr_+$, given by 
\[ h(M)= C_1 \sqrt{M} + \frac{C_2}{M^2} \;, \quad C_1 >0\;, C_2>0  \;. \]
A short calculation shows that the optimum is achieved at 
\[ M_{opt} = \left( \frac{4C_2}{C_1} \right)^{2/5}\;, \]
with value 
\[ h(M_{opt}) = 5C_2 \left( \frac{C_1}{4C_2} \right)^{4/5}= 5C_2 \frac{1}{ M^2_{opt}} \;.\] 
Setting now 
\[ C_1 := c\sigma_x \; \frac{\SNR \cdot F}{d\sqrt{n}}   \;, \quad    C_2 :=   \tilde c \tau^2 C_{\rho_{2,n}} \; \frac{n}{F}  \]
gives for the optimal number of local nodes  
\[  M_{opt} = M_n = A\; \left(  \frac{d n^{3/2}}{\SNR \cdot F^2} \right)^{2/5}\;, \quad 
       A:= c'\left(\frac{1}{1-\rho_{2,n}}   \right)^{4/5} \;, \]
where $c'=(4 \tilde c \tau^2)/(c \sigma_x)$ and 
\begin{align*}
\mbe_{\beta^* ,\epsilon} \left[||\Sigma^{\frac{1}{2}}(\bar \beta_{M_n} - \beta^*) ||^2 \right] 
&\leq 5 \log^{\frac{1}{2}}\left(\frac{2M_n}{\delta}\right) \; C_2 \left( \frac{C_1}{4C_2} \right)^{4/5} \\
&= \tilde c'\; C^{3/5}_{\rho_{2,n}} \log^{\frac{1}{2}}\left(\frac{2M_n}{\delta}\right) \;  \frac{n}{F M_n^2} \\
&\leq c''  \log^{\frac{1}{2}}\left(\frac{2M_n}{\delta}\right) \;
 \left( \frac{\SNR}{d (1-\rho_{2,n})^2}\right)^{4/5} \;\left( \frac{F}{d}\right)^{4/5}\left( \frac{1}{nF}  \right)^{1/5} \;, 
\end{align*}
where $\tilde c'=5 \tilde c \tau^2$ and  $c''=2(c\sigma_x)^{4/5}\cdot(\tilde c \tau^2)^{1/5}$.

Note that this bound only makes sense if $1 \lesssim M_n \lesssim n$. A short calculation shows that 
\[ 1 \lesssim M_n \mbox{ if } \frac{d}{F^2} \gtrsim \frac{\SNR}{n^{3/2}} \]
and 
\[  M_n \lesssim \mbox{ if }   \frac{d}{F^2} \lesssim  \SNR\; n  \;. \]
\end{proof}

\[\]


\subsection{Proofs of Section \ref{subsec:lower-bound}}
\label{app:prof-lower-bound}

We first recall a lower bound for the variance in the single machine setting. 

\begin{proposition}[Lemma 10 and Lemma 11 in \cite{bartlett2020benign}]
\label{prop:lower-var-single}
Define 
\[ \widehat{C}:= (\bX\bX^T)^{-1}\bX \Sigma \bX^T(\bX\bX^T)^{-1} \;.  \] 
There exists a constant $c>0$ such that for any $0 \leq k \leq n/c$ and any $a >1$ 
with probability at least $1-10e^{-n/c}$, if  $r_k(\Sigma) \geq a\; n$, then 
\[ \tr[\widehat{C}] \geq \frac{1}{ca} \min_{l \leq k} \left( \frac{l}{n} + \frac{a^2 n \sum_{j > l}\lam_j^2}{(\lam_{k+1} r_k(\Sigma))^2}  \right) \;.\]
Moreover, for 
\[ k^* := \min\left\{ k: r_k(\Sigma) \geq a \; n \right\} \]
and if $k^*< \infty$, then  
\[ \min_{l \leq k^*} \left( \frac{l}{a \;n} + \frac{a \; n \sum_{j > l}\lam_j^2}{(\lam_{k+1} r_k(\Sigma))^2}  \right) 
=  \frac{k^*}{a \; n} + \frac{a\; n}{R_{k^*}(\Sigma)} \;.   \]
\end{proposition}

\vspace{0.2cm}

\begin{proof}[Proof of Theorem \ref{theo:lower-bound}]
From Lemma \ref{lem:bias-var} and its proof, in particular by \eqref{eq:tr1}, and by  Assumption \ref{ass:lower-noise} we may lower bound the excess risk by 
the variance and find  
\begin{align*}
\mbe_\eps[ ||\Sigma^{1/2}( \bar \beta_M - \beta^*) ||^2 ] &\geq 
\frac{1}{M^2}  \sum_{m=1}^M \mbe_\eps\Big[  \tr\Big[  \Sigma^{1/2}  X_m^\dagger \eps_m \otimes \eps_{m} (X_{m}^\dagger)^T \Sigma^{1/2} \Big]  \Big] \\
&= \frac{1}{M^2}  \sum_{m=1}^M  \tr\Big[  \Sigma^{1/2}  X_m^\dagger \mbe_\eps [\eps_m \otimes \eps_{m} ](X_{m}^\dagger)^T \Sigma^{1/2} \Big] \\
&\geq \frac{\sigma^2}{M^2}  \sum_{m=1}^M  \tr\Big[ C_m \Big]  \;, 
\end{align*} 
where $C_m=(\bX_m^\dagger)^T\Sigma\bX_m^\dagger$. 
Recall that by definition of $k^*_{\frac{n}{M}}$ from Definition \ref{def:eff-dim} we have 
$r_{k^*_{\frac{n}{M}}}(\Sigma) \geq a \frac{n}{M}$. Hence, we may apply Proposition \ref{prop:lower-var-single} 
and obtain with probability at least $1-10Me^{-\frac{1}{c}\frac{n}{M}}$ 
\begin{align*}
\mbe_\eps[ ||\Sigma^{1/2}( \bar \beta_M - \beta^*) ||^2 ] &\geq  
\frac{\sigma^2}{caM^2}  \sum_{m=1}^M \left(  \frac{ M k^*_{\frac{n}{M}}}{n} + \frac{a^2 n}{M} \frac{ \sum_{j > k^*}\lam_j^2}{(\lam_{k^*+1} r_{k^*}(\Sigma))^2}    \right) \\
&= \frac{\sigma^2}{ca} \left(  \frac{k^*_{\frac{n}{M}}}{n} + \frac{a^2 n}{M^2} \frac{ \sum_{j > k^*}\lam_j^2}{(\lam_{k^*+1} r_{k^*}(\Sigma))^2}    \right) \\
&\geq c_a\sigma^2 \left(  \frac{k^*_{\frac{n}{M}}}{n} + \frac{a^2 n}{M^2} \frac{ \sum_{j > k^*}\lam_j^2}{(\lam_{k^*+1} r_{k^*}(\Sigma))^2}    \right) \;,
\end{align*} 
where we set $c_a := \frac{1}{ca}$ and use that $a>1$. 
\end{proof}


\vspace{0.2cm}

\begin{proof}[Proof of Corollary \ref{cor:opt-poly}]
The proof of Lemma \ref{lem:poly} shows that $r_k(\Sigma) \geq \frac{k+1}{\eps_n}$,  for any $k$. A similar calculation gives as upper bound 
\begin{align*}
r_k(\Sigma) &= (k+1)^{1+\eps_n} \sum_{j >k} j^{-(1+\eps_n)} \\
&\leq \int_k^\infty t^{-(1+\eps_n)} dt \\
&= \frac{(k+1)^{1+\eps_n}}{\eps_n k^{\eps_n}} \\
&\leq \frac{ (2k)^{1+\eps_n}}{\eps_n k^{\eps_n}} \\
&\leq \frac{4 k}{\eps_n}\;,
\end{align*}  
where we use that $1 \leq k$ and $2^{1+\eps_n} \leq 4$, since $\eps_n \leq 1$ for $n$ sufficiently large. In particular, 
\begin{equation}
\label{eq:equiv}
 a \frac{n}{M} \leq \frac{k^* + 1}{\eps_n} \leq r_{k^*}(\Sigma) \leq \frac{4k^*}{\eps_n}\;. 
\end{equation} 
Thus, $k^* \geq \frac{a}{4}\frac{n}{M} \eps_n$. 

{\bf High Smoothness $\alpha>0$.}
Moreover, the definition of $M_n$ in Corollary \ref{cor:poly1} gives 
\[ \frac{k^*_{\frac{n}{M_n}}}{n} \geq \frac{a}{4} \frac{\eps_n}{M_n} =  \frac{a}{4C_{\tau, \sigma_x}} \alpha^{-2/3} \left( \frac{1}{n}\right)^{1/3} \;. \]  
Hence, applying Theorem \ref{theo:lower-bound} gives with probability at least $1-10Me^{-\frac{1}{c}\frac{n}{M}}$ 
\begin{align}
\label{eq:lower-high}
\mbe_\eps[ ||\Sigma^{1/2}( \bar \beta_{M_n} - \beta^*) ||^2 ] 
&\geq c_a\sigma^2 \left(  \frac{k^*_{\frac{n}{M_n}}}{n} + \frac{a^2 n}{M_n^2} \frac{ \sum_{j > k^*}\lam_j^2}{(\lam_{k^*+1} r_{k^*}(\Sigma))^2}    \right) \no \\
&\geq c_a\sigma^2 \frac{k^*_{\frac{n}{M_n}}}{n} \\
&\geq \frac{\tilde C_{\tau, \sigma_x , \sigma}}{\alpha^{2/3} } \left( \frac{\eps_n}{n}\right)^{1/3} \;, \no
\end{align}
with $\tilde C_{\tau, \sigma_x , \sigma} = \frac{a c_a \sigma^2}{4C_{\tau, \sigma_x}}$. 

{\bf Low Smoothness $\alpha=0$.} 
The result in this regime follows from the same arguments as above by inserting the definition of $M_n$ in Corollary \ref{cor:poly2} in the above equations. 
Indeed, one easily finds with \eqref{eq:equiv} and \eqref{eq:lower-high} 
\begin{align*}
\mbe_\eps[ ||\Sigma^{1/2}( \bar \beta_{M_n} - \beta^*) ||^2 ] 
&\geq c_a\sigma^2 \frac{k^*_{\frac{n}{M_n}}}{n} \\
&\geq \tilde C_{\tau, \sigma_x , \sigma} \left( \frac{1}{\eps_n n}\right)^{1/3} \;,
\end{align*}
for some $\tilde C_{\tau, \sigma_x , \sigma} >0$. 
\end{proof}

\vspace{0.2cm}

\begin{proof}[Proof of Corollary \ref{cor:opt-finit-dim}]
Applying Theorem \ref{theo:lower-bound} gives with probability at least $1-10Me^{-\frac{1}{c}\frac{n}{M}}$ 
\begin{align*}
\mbe_\eps[ ||\Sigma^{1/2}( \bar \beta_{M} - \beta^*) ||^2 ] 
&\geq c_a\sigma^2 \left(  \frac{k^*_{\frac{n}{M}}}{n} + \frac{a^2 n}{M^2} \frac{ \sum_{j > k^*}\lam_j^2}{(\lam_{k^*+1} r_{k^*}(\Sigma))^2}    \right) \no \\
&\geq c_a\sigma^2  \frac{a^2 n}{M^2} \frac{ \sum_{j > k^*}\lam_j^2}{\left( \sum_{j > k^*}\lam_j \right)^2} \;. 
\end{align*}
From \eqref{eq:squared-EV} we have 
\begin{align*}
\sum_{j=1}^d  \lam_j^2(\Sigma)  &= \theta^4\left( (\rho_1^2-\rho_2^2)F + \rho_2^2 d \right) \\
&\geq \theta^4 (\rho_1^2-\rho_2^2)F \;.
\end{align*}
Moreover, \eqref{eq:EV-not-squared} gives with $\rho_2 d \leq F$ 
\begin{align*}
\sum_{j=1}^d  \lam_j(\Sigma)  &= \theta^2\left( (\rho_1-\rho_2 )F + \rho_2 d \right) \\
&\leq \theta^2\left( \rho_1-\rho_2 + 1 \right) F \;.
\end{align*}
Hence, 
\begin{align*} 
\mbe_\eps[ ||\Sigma^{1/2}( \bar \beta_{M} - \beta^*) ||^2 ] 
&\geq c_a\sigma^2  \frac{a^2 n}{F M^2}  \frac{\rho^2_1-\rho_2^2}{ (\rho_1-\rho_2 + 1)^2 }   \\
&= \tilde c_a\sigma^2  \tilde C_{\rho_1, \rho_2}  \frac{ n}{F M^2} \;,
\end{align*}
where we set $\tilde c_a = c_a a^2$ and $\tilde C_{\rho_1, \rho_2}=  \frac{\rho^2_1-\rho_2^2}{ (\rho_1-\rho_2 + 1)^2 }  $.
\end{proof}


\section{Additional Results}
\label{sec:add-results}


In this section we collect some additional results. We first analyze the finite dimensional setting under a more 
general \emph{source condition} and investigate the impact of the hardness of the problem on the number of optimal machines. 
In addition, we give a general lower bound in finite dimension under general distributional assumptions.

\subsection{General source condition in the strong-weak-features model (finite dimension)}
\label{sec:finite-dim-normal-sparsity}

In this section we analyze the setting of Section \ref{subsec:sparse-finite-dim} under a more general prior assumption. 
Here, the covariance of $\beta^*$ will have a specific structure, described by a \emph{source function} 
$\Phi: \mbr_+ \to \mbr_+$, with $t \mapsto t\Phi(t)$ non-decreasing.

\begin{assumption}[Source Condition]
\label{ass:source2}
Assume that $\beta^* \sim \cN(0, \frac{R^2}{d} \Phi(\Sigma))$, for some $R>0$. 
Note that $R^2$ can be intepreted as the \emph{expected signal strength}. 
\end{assumption}

\vspace{0.1cm}

\begin{lemma}
\label{lem:source-finite}
Suppose Assumption \ref{ass:stron-weak} is satisfied. Let $\rho_2\Phi(\rho_2)d \leq F$. 
Then 
\[ \tr[ \Sigma  \Phi(\Sigma)) ]  \; \leq \; (\rho_1\Phi(\rho_1) + 1)F  \;. \]
\end{lemma}

\vspace{0.1cm}

\begin{proof}[Proof of Lemma \ref{lem:source-finite}]
We write 
\begin{align*}
\tr[ \Sigma  \Phi(\Sigma)) ]  &= \sum_{j=1}^F \rho_1\Phi(\rho_1) + \sum_{j=F+1}^d \rho_2\Phi(\rho_2) \\
&= (\rho_1\Phi(\rho_1)- \rho_2\Phi(\rho_2) )F + \rho_2\Phi(\rho_2) d \\
&\leq  (\rho_1\Phi(\rho_1) - \rho_2\Phi(\rho_2) + 1)F \\
&\leq  (\rho_1\Phi(\rho_1)  + 1)F \;.
\end{align*}
\end{proof}

\vspace{0.2cm}

\begin{proposition}
\label{prop:finite1}
In addition to all assumptions of Theorem \ref{theo:main1}, suppose that Assumptions \ref{ass:source2} and \ref{ass:stron-weak} are satisfied. 
Assume that $\rho_2\Phi(\rho_2)d \leq F$. 
For any $\delta \geq 2e^{-c_1^2\frac{n}{M}}$, with probability at least $1-\delta-7Me^{-\frac{n}{c_2 M}}$ we have  
\begin{align*}
  \mbe_{\beta^* ,\epsilon} \left[||\Sigma^{\frac{1}{2}}(\bar \beta_M - \beta^*) ||^2 \right]  
& \leq c_3 C_{\rho_1}  \log^{\frac{1}{2}}\left(\frac{2M}{\delta}\right) \frac{R^2F}{d} \sqrt{\frac{M}{n}} 
  +  c_4  \;  \Delta^{-1}(\rho_1, \rho_2)\; \frac{n}{M^2}\;\frac{1}{F}  \;, 
\end{align*}
where $C_{\rho_1}=\rho_1\Phi(\rho_1)+1$ and $\Delta(\rho_1, \rho_2) := (\rho_1-\rho_2)^2$ and for some $c_1, c_2, c_3,c_4 >0$. 
\end{proposition}

\vspace{0.1cm}

\begin{proof}[Proof of Proposition \ref{prop:finite1}]
We combine Theorem \ref{theo:main1}, Lemma \ref{lem:var-strong-wak} and Lemma \ref{lem:source-finite}. 
This gives with probability at least $1-\delta-7Me^{-\frac{n}{c_2 M}}$
\begin{align*}
 \mbe_{\beta^* ,\epsilon} \left[||\Sigma^{\frac{1}{2}}(\bar \beta_M - \beta^*) ||^2 \right] &\leq 
\frac{4 \sigma_x}{c_1} \log^{\frac{1}{2}}\left(\frac{2M}{\delta}\right) \tr[ \Sigma \Theta]   \sqrt{\frac{M}{n}} + 
16c_2 \tau^2 \;  \frac{ 1}{(\rho_1-\rho_2)^2} \; \frac{n}{M^2}\;\frac{1}{F} \\
&\leq \frac{4 \sigma_x}{c_1}(\rho_1\Phi(\rho_1)  + 1) \frac{R^2F}{d} \log^{\frac{1}{2}}\left(\frac{2M}{\delta}\right)  \sqrt{\frac{M}{n}} + 
16c_2 \tau^2 \;  \frac{ 1}{(\rho_1-\rho_2)^2} \; \frac{n}{M^2}\;\frac{1}{F} \;.
\end{align*}
The results follows by setting $C_{\rho_1}:=\rho_1\Phi(\rho_1)  + 1$, $\Delta(\rho_1, \rho_2) := \Delta:= (\rho_1-\rho_2)^2$, $c_3:=4 \sigma_x/c_1$ and $c_4:=16c_2\tau^2$.
\end{proof}

\vspace{0.1cm}

\begin{corollary}[Optimal number of machines]
\label{cor:opt-nodes-normal}
Suppose all assumptions of Proposition \ref{prop:finite1} are satisfied. 
Let $(\rho_{2,n})_n$ be decreasing and $\rho_{2,n}\Phi(\rho_{2,n})d_n \leq F_n$. 
Denote $\Delta_n:=(\rho_1 - \rho_{2,n})^2$ and assume 
\begin{equation}
\label{eq:ass-f-d-3}
 n^{-3/2} \lesssim \frac{d_n}{\Delta_n F_n^2} \lesssim n  \;.  
\end{equation} 
The optimal 
number\footnote{The optimal number is defined as the minimizer of the right hand side in $\eqref{eq:rhs}$, where we ignore the log-factor.} of local 
nodes $M_n$ is given by 
\begin{equation}
\label{eq:nb-machnines-3}
 M_n=A\cdot\left(  \frac{ d_n n^{3/2}}{R^2 \Delta_n \cdot F_n^2} \right)^{2/5}\;, 
\end{equation} 
where $A = \left(\frac{4c_4}{c_3C_{\rho_1}}\right)^{2/5}$.  
The excess risk satisfies  with probability at least $1-\delta-7M_n e^{-\frac{n}{c_2 M_n }}$
\begin{align}
\label{eq:rate-finite-3}
\mbe_{\beta^* ,\epsilon} \left[||\Sigma^{\frac{1}{2}}(\bar \beta_{M_n} - \beta^*) ||^2 \right] 
&  \leq  c'  \log^{\frac{1}{2}}\left(\frac{2M_n}{\delta}\right) \;
 \left( \frac{R^2 F_n}{d_n} \right)^{4/5} \;\left( \frac{1}{F_n \cdot n \Delta_n } \right)^{1/5} \;,
\end{align}
where $c' =(5 c_4)/A^2$.
\end{corollary}

\vspace{0.1cm}

\begin{proof}[Proof of Corollary \ref{cor:opt-nodes-normal}]
We need to determine the minimum of the function $h: \mbr_+ \to \mbr_+$, given by 
\[ h(M)= C_1 \sqrt{M} + \frac{C_2}{M^2} \;, \quad C_1 >0\;, C_2>0  \;. \]
A short calculation shows that the optimum is achieved at 
\[ M_{opt} = \left( \frac{4C_2}{C_1} \right)^{2/5}\;, \]
with value 
\[ h(M_{opt}) = 5C_2 \left( \frac{C_1}{4C_2} \right)^{4/5}= 5C_2 \frac{1}{ M^2_{opt}} \;.\] 
Setting now 
\[ C_1 := c_3\;C_{\rho_1} \frac{R^2 F}{d\sqrt{n}}   \;, \quad    C_2 :=   c_4  \; \frac{n}{\Delta_n \cdot F}  \]
gives for the optimal number of local nodes  
\[  M_{opt} = M_n = A\cdot\left(  \frac{ d n^{3/2}}{R^2 \Delta_n \cdot F^2} \right)^{2/5}\;, \quad 
       A:= \left(\frac{4c_4}{c_3C_{\rho_1}}\right)^{2/5} \;, \]
and 
\begin{align*}
\mbe_{\beta^* ,\epsilon} \left[||\Sigma^{\frac{1}{2}}(\bar \beta_{M_n} - \beta^*) ||^2 \right] 
&\leq 5c_4 \log^{\frac{1}{2}}\left(\frac{2M_n}{\delta}\right) \;  \frac{n}{F \cdot \Delta_n  M_n^2} \\
&= c'  \log^{\frac{1}{2}}\left(\frac{2M_n}{\delta}\right) \;
 \left( \frac{R^2 F}{d} \right)^{4/5} \;\left( \frac{1}{F \cdot n \Delta_n } \right)^{1/5} \;, 
\end{align*}
where $c'=(5 c_4)/A^2$.  Moreover, for our bounds to be meaningful we have to require that $1 \lesssim M_n \lesssim n$. This is satisfied if 
\[ n^{-3/2} \lesssim \frac{d_n}{\Delta_n F_n^2} \lesssim n \;.  \] 
\end{proof}

\vspace{0.2cm}

The two conditions 
\begin{itemize}
\item[(I)] $n^{-3/2} \lesssim \frac{d_n}{\Delta_n F_n^2} \lesssim n$
\item[(II)] $\rho_{2,n}\Phi(\rho_{2,n})d_n \leq F_n$
\end{itemize}
from  Corollary \ref{cor:opt-nodes-normal} determine the number of optimal splits and the learning rate of the distributed minimum norm interpolant. 
In particular, the a-priori assumption on $\beta^*$ through the source function $\Phi$ has an influence on the possible number of splits and hence 
on the efficiency of averaging. We discuss three special examples in more detail below.   In all cases, we exclusively focus on the overparameterized 
regime where $n \lesssim d_n$ and $1\lesssim F_n \lesssim d_n$. 
Suppose that 
\[ d_n \simeq  n^\gamma \;, \quad \gamma >1 \mbox{\; and \; } F_n \simeq  n^\delta \;, \quad 0 \leq \delta \leq \gamma \;.  \]
Condition $(II)$ from above sets now restrictions on the decay of the strength of the weak features.

\vspace{0.2cm}

\paragraph{Easy Case.} We let $\Phi(t)=t$.  Condition $(II)$ can be rewritten as $\rho_{2, n} \lesssim \left( \frac{1}{n} \right)^{\frac{1}{2}(\gamma - \delta)}$. 
To meet condition $(I)$ we need to distinguish two cases: 
\begin{itemize}
\item If $\gamma \leq 2\delta$, we have $\max\{ 1, 2\delta - 3/2\} < \gamma \leq 2\delta$ and $\delta >1/2$. In particular, the number of strong features 
needs to grow at as $F_n \gtrsim \sqrt{n}$. 
\item If $\gamma \geq 2\delta$, we have $\max\{ 1, 2\delta \} < \gamma \leq 2\delta + 1$ and $\delta < \gamma/2$. Here, the number of strong features 
can not grow faster that $n^{\gamma/2}$. 
\end{itemize}

\paragraph{Isotropic Case.} We let $\Phi(t)=1$. Condition $(II)$ can be rewritten as $\rho_{2, n} \lesssim \left( \frac{1}{n} \right)^{\gamma - \delta}$. 
Compared to the \emph{easy case}, the strength of the weak features $\rho_{2,n}$ needs to decay faster. Condition $(I)$ 
holds under the same assumptions as in the \emph{easy case}.

\paragraph{Hard Case.} We let $\Phi(t)=t^{-1}$. 
Condition $(II)$ reduces to $F_n \simeq d_n$, i.e., the number of strong features needs to grow as fast as the dimension. 
In this case, the optimal number of machines scales as 
$M_n \simeq n^{\frac{2}{5}(\frac{3}{2}-\delta)}$. To ensure $1 \lesssim M_n \lesssim n$, 
the growth of $d_n$ can not be too fast: $\gamma = \delta \in (1, 3/2]$.

\vspace{0.3cm}




\subsection{A universal lower bound  (finite dimension)}
\label{app:universal-lower}

We aim at deriving a lower bound for the distributed ridgeless regression estimator under fairly general distributional 
assumptions if $dim (\cH) = d < \infty$. 

\begin{assumption}
\label{ass:general-lower-bound}
\begin{enumerate}
\item The input $x \in \mbr^d$ is strongly square integrable: $\mbe[||x||^2]<\infty$.  
\item The covariance matrix $\Sigma \in \mbr^{d \times d}$ is invertible. 
\item $\mbe[y^2] < \infty$.  
\item The conditional variance is bounded from below: For some $\tilde \tau  \geq 0$ we assume $\mbv[y|x] \geq \tilde \tau^2$ almost surely. 
\item For any $m=1, ..., M$, the local data matrix $X_m \in \mbr^{ \frac{n}{M} \times d}$ has almost surely full rank, i.e., $\rank[ X_m] = \min\{ \frac{n}{M} ,d\}$.
\end{enumerate}
\end{assumption}

\vspace{0.1cm}

Under these assumptions we have the following lower bound for the ridgeless distributed estimator in finite dimension.

\vspace{0.1cm}
\begin{theorem}[Lower bound]
\label{theo:rough-lower}
Let $\bar \beta_{M}$ be defined by \eqref{eq:weighted-ave}. The excess risk satisfies  
\[ 
\mbe \left[ ||\Sigma^{1/2}(\bar \beta_{M} -\beta^*)||^2 \right]
\geq  
 \frac{\tilde \tau^2}{M} \; \frac{\min \{ d , \frac{n}{M} \}}{ \max\{d, \frac{n}{M} \} + 1 - \min\{d, \frac{n}{M}\}}  \;.  \]
\end{theorem}

Thus,  we observe peaks at $d= \frac{n}{M}$ with height at least $\tilde \tau^2 \frac{d}{M}$, see Fig. {fig:3}.

\vspace{0.2cm}

We consider functions of the form $f_\beta : \cH \to \mbr$, $\beta \in \cH$, with $f_\beta (x) := \inner{\beta , x}$ 
and  define for any estimator $\hat \beta \in \cH$ the quantity
\[ \tilde \cE:= \mbe[ ( f_{\hat \beta} (x) -  \mbe_{Y|X}[f_{\hat \beta} (x)] )^2  ]  \;.  \]
One easily verifies that 
\begin{equation}
\label{eq:tildeE}
  \tilde \cE \leq \mbe[  \cR(\hat \beta)] - \cR(\beta^*)  \;.
\end{equation}  
Thus, finding a lower bound for $\tilde \cE$ leads to a lower bound for the excess risk. 

\vspace{0.2cm}

\begin{proof}[Proof of Theorem \ref{theo:rough-lower}]
Define the centered output variables $\tilde Y_m := Y_m - \mbe_{Y_m|X}[Y_m]$, $m=1, ..., M$ and 
set 
\[ Cov(Y_m , X_m):= \mbe_{Y_m|X_m}[ \tilde Y_m  \otimes \tilde Y_m ]  \;. \]
We then write 
\begin{align*}
 \tilde \cE(X) &:= \mbe_{Y,x|X}[ ( f_{\bar \beta_0} (x) -   \mbe_{Y|X}[f_{\bar \beta_0} (x) ] )^2 ] \\
 &= \mbe_{Y,x|X}\left[ \left( \inner{ x , \frac{1}{M}\sum_{m=1}^M   X_m^\dagger \tilde Y_m }\right)^2 \right] \\
 &= \mbe_{Y,x|X}\left[ \left( \frac{1}{M}\sum_{m=1}^M \inner{ x ,  X_m^\dagger \tilde Y_m } \right)^2 \right] \\
&= \frac{1}{M^2}\sum_{m=1}^M \sum_{m'=1}^M   \mbe_{Y,x|X}\left[ \inner{ x ,  X_m^\dagger \tilde Y_m }\inner{ x ,  X_{m'}^\dagger \tilde Y_{m'} }     \right]  \;. 
\end{align*} 
Note that by definition of $\tilde Y_m$ and linearity we have 
\[ \mbe_{Y,x|X}\left[  \inner{ x ,  X_m^\dagger \tilde Y_m } \right] = 0 \;. \]
Thus, by independence and Assumption \ref{ass:general-lower-bound} we find 
\begin{align}
\label{eq:first-low}
 \tilde \cE(X) &=   \frac{1}{M}\sum_{m=1}^M    \mbe_{Y,x|X}\left[ \inner{ x ,  X_m^\dagger \tilde Y_m }^2  \right] \nonumber  \\
 &= \frac{1}{M}\sum_{m=1}^M     \mbe_{Y,x|X}\left[ \inner{ x ,  X_m^\dagger \tilde Y_m }^2  \right] \nonumber  \\
 &= \frac{1}{M} \sum_{m=1}^M  \mbe_x \left[  \inner{ ( X_m^\dagger )^T x, Cov(Y_m , X_m) (X_m^\dagger )^T x }  \right] \nonumber  \\
 & \geq  \frac{\tilde \tau^2}{M} \sum_{m=1}^M  \mbe_x \left[    || ( X_m^\dagger )^T x ||^2\right] \nonumber \\
 &=  \frac{\tau^2}{M}\sum_{m=1}^M  Tr \left[ ( X_m^\dagger )^T \mbe_x \left[  x \otimes x\right]  X_m^\dagger    \right] \nonumber  \\
 &=  \frac{\tau^2}{M}\sum_{m=1}^M  Tr \left[ ( X_m^\dagger )^T \Sigma  X_m^\dagger   \right] \;.
\end{align} 
We proceed by introducing the whitened data matrices 
\[ W_m:= X_m \Sigma^{-1/2}\;.  \]
We then 
distinguish the two cases: 
\vspace{0.2cm}
\\
{\bf (I) $d\geq b=\frac{n}{M}$: } Following the arguments in \cite{holzmuller2020universality} (Proof of Theorem 3) shows that 
\[  \mbe_{X}\left[ Tr \left[ ( X_m^\dagger )^T \Sigma  X_m^\dagger   \right] \right]  \geq 
  \mbe_{X}\left[  Tr \left[ (W_m W_m^T)^{-1}  \right]  \right] \geq \frac{b}{d+1-b}  \;. \]
Combining this with \eqref{eq:first-low} gives by independence 
\begin{align*}
\tilde \cE &=  \mbe_{X}[\tilde \cE(X)] \\
&\geq  \frac{\tilde \tau^2}{M} \frac{b}{d+1-b} \\ 
&= \frac{\tilde \tau^2}{M}\frac{b}{d+1-b}  \;. 
\end{align*}
The result follows from \eqref{eq:tildeE}. 
\vspace{0.2cm}
\\
{\bf (II) $d\leq b=\frac{n}{M}$: } A short calculation shows that 
\[   Tr \left[ ( X_m^\dagger )^T \Sigma  X_m^\dagger   \right] =  Tr \left[ (W_m^T W_m)^{-1}  \right] \;. \] 
Following again  \cite{holzmuller2020universality} (Proof of Theorem 3) we readily obtain
\[  \mbe_{X}\left[ Tr \left[ ( X_m^\dagger )^T \Sigma  X_m^\dagger   \right] \right]  = 
  \mbe_{X}\left[  Tr \left[ (W_m^T W_m )^{-1}  \right]  \right] \geq \frac{d}{b+1-d}  \;. \] 
We conclude as above to obtain the result. 
\end{proof}

\vspace{0.3cm}


\section{Additional Numerical Results}
\label{app:add-numerics}


\begin{figure}[ht]
\label{fig:3}
\centerline{
\includegraphics[width=0.55\columnwidth, height=0.35\textheight]{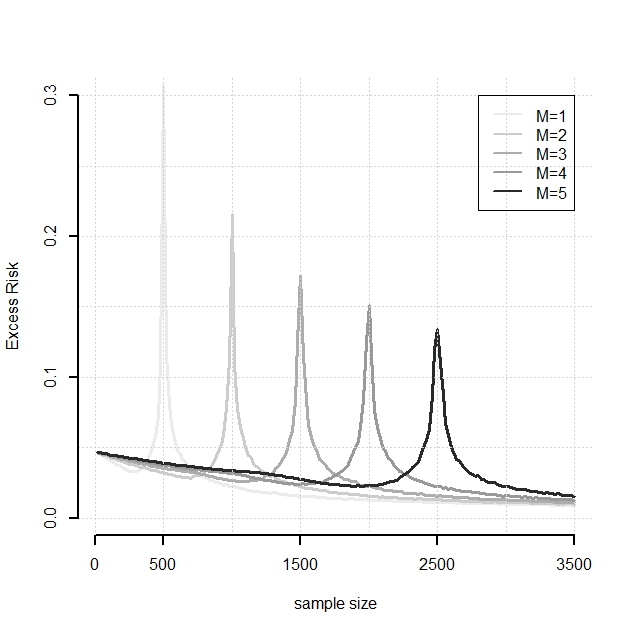} 
}
\caption{ Double descent for MSDYear dataset from Section \ref{sec:numerics} with $d=90$ features. We observe peaks whenever $d=\frac{n}{M}$, as 
Theorem \ref{theo:rough-lower} predicts.  
}
\end{figure}

\paragraph{Simulated Data.}

In a final experiment we investigate the effect of decay of the eigenvalues on the (normalized) relative prediction efficiency, 
defined in Definition \ref{def:rel-eff-intro}. We generate $n=200$ i.i.d. training points $x_j \sim \cN(0, \Sigma)$, with $d=400$, 
$\lam_j(\Sigma)=j^{-(1+\eps)}$, with $\eps=0.1, 0.5, 1, 1.5$. 
The target $\beta^*$ is simulated according to Assumption \ref{ass:sparse} with $\SNR=1$. 
As expected from our main results, faster decay (larger $\eps$) allows larger parallelization, that is, the optimal number of splits (largest efficiency) 
increases with faster decay.

\begin{figure}[ht]
\label{fig:5}
\centerline{
\includegraphics[width=0.45\columnwidth, height=0.25\textheight]{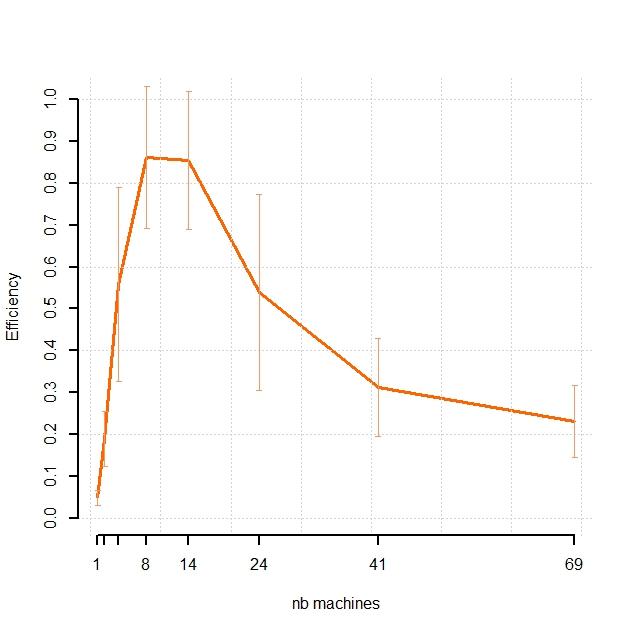}\hspace{0.2cm}
\includegraphics[width=0.45\columnwidth, height=0.25\textheight]{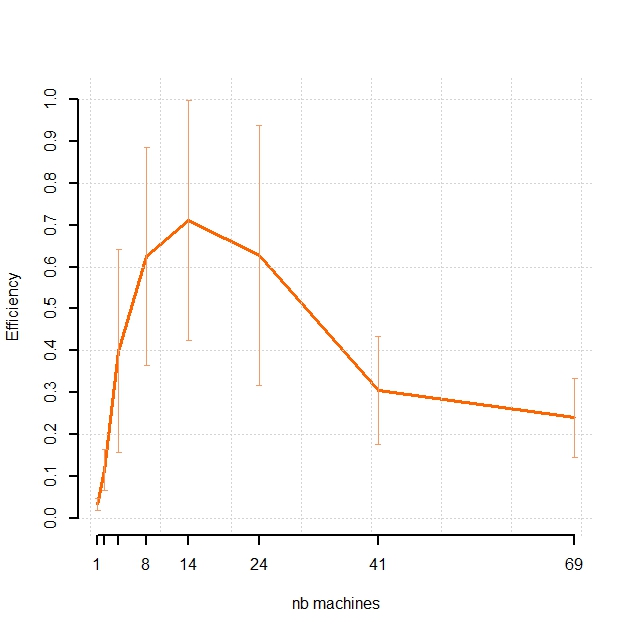} 
}
\caption{{\bf Left:} $\eps=0.1$ {\bf Right:} $\eps=0.5$.
}
\end{figure}

\begin{figure}[ht]
\label{fig:4}
\centerline{
\includegraphics[width=0.45\columnwidth, height=0.25\textheight]{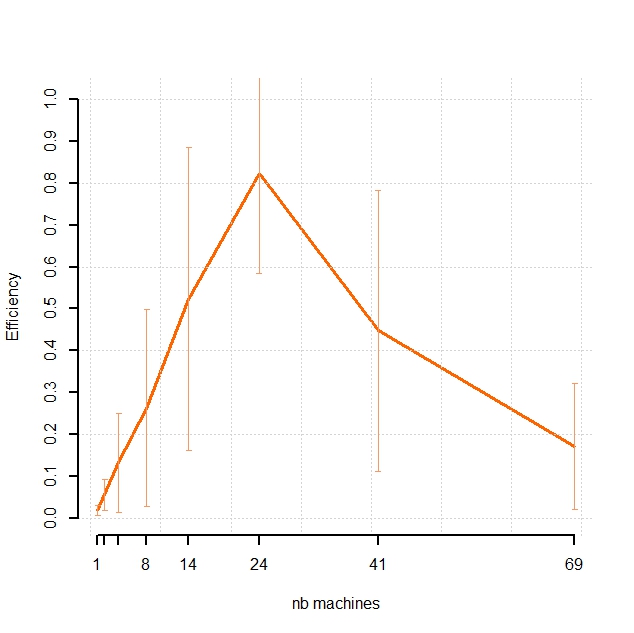} \hspace{0.2cm}
\includegraphics[width=0.45\columnwidth, height=0.25\textheight]{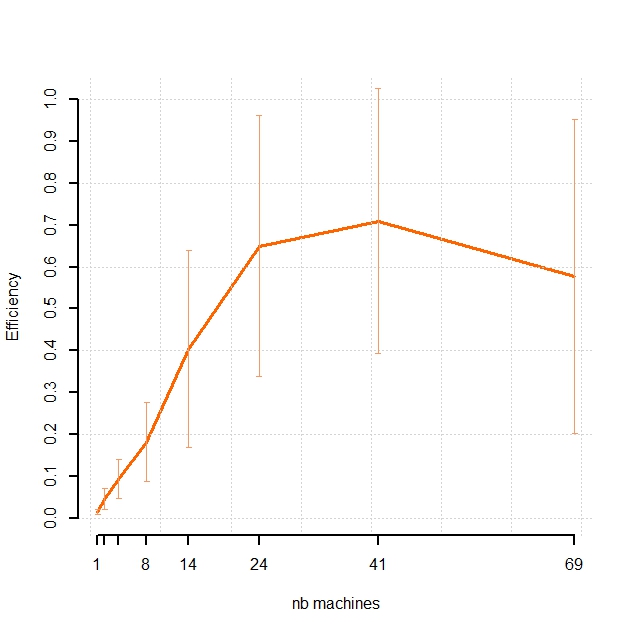}
}
\caption{{\bf Left:} $\eps=1$ {\bf Right:} $\eps=1.5$. 
}
\end{figure}



\checknbnotes

\end{document}